\documentclass{article} 
\usepackage{iclr2027_conference,times}
\usepackage{natbib}

\usepackage{amsmath,amsfonts,bm}

\def\eqref#1{equation~\ref{#1}}

\def\1{\bm{1}}

\DeclareMathAlphabet{\mathsfit}{\encodingdefault}{\sfdefault}{m}{sl}
\SetMathAlphabet{\mathsfit}{bold}{\encodingdefault}{\sfdefault}{bx}{n}

\usepackage{bm}
\usepackage{amsmath}
\usepackage{amssymb}
\usepackage{mathtools}
\usepackage{amsthm}
\usepackage{multirow}
\usepackage{pifont}  
\usepackage[table]{xcolor}
\usepackage{threeparttable}

\usepackage{algorithm}
\usepackage{algorithmic}
\usepackage{listings}
\usepackage{float}
\usepackage{caption}
\usepackage{booktabs}
\usepackage{hyperref}
\usepackage{url}
\theoremstyle{plain}
\newtheorem{theorem}{Theorem}[section]
\newtheorem{proposition}[theorem]{Proposition}

\newtheorem{corollary}[theorem]{Corollary}
\theoremstyle{definition}

\theoremstyle{remark}

\title{ProtoFlow: Prototype-Guided Flow Matching for Multivariate Time Series Forecasting}

\author{Shibo Feng$^{1}$, \quad Wanjin Feng$^{2}$, \quad Yang Qiu$^{1}$, \quad Deheng Ye$^{1}$, \quad Peilin Zhao$^{3, *}$, \quad Chunyan Miao $^{1}$\thanks{Corresponding author} \\
 $^{1}$Nanyang Technological University, \quad $^{2}$ Tsinghua
University, \quad $^{3}$Shanghai Jiao Tong University \\
        }

\iclrfinalcopy 
\begin{document}

\maketitle
\lhead{Preprint}
\begin{abstract}
Generative modeling has shown strong promise for multivariate time mseries (MTS) forecasting, especially scale to high-dimensional settings. Diffusion-based methods achieve competitive performance but typically require many sampling steps at inference. VAE-based non-iterative forecasting frameworks have therefore emerged as an efficient alternative. Within this line of work, vector quantization (VQ) enables controllable latent space modeling by mapping multivariate series into compact discrete representations. Existing VQ-based forecasting methods, however, typically rely on autoregressive (AR) token generation, which suffers from exposure bias and training-inference mismatch. Flow matching provides an efficient non-autoregressive alternative for latent forecasting, but existing formulations usually initialize transport from a generic Gaussian prior. We instead observe that the trained VQ codebook already captures representative latent prototypes and can thus serve as a more informative prior for flow matching. Based on this insight, we propose \textbf{ProtoFlow}, a forecasting framework that combines vector-quantized autoencoding with \textbf{Proto}type-prior \textbf{Flow} matching. Our method first maps multivariate sequences into a discrete latent space, then constructs a structured prior from the learned codebook, and finally learns a DiT-based rectified flow to transport samples from this prior to future latent representations conditioned on historical observations. By replacing generic noise initialization with a learned prototype prior, ProtoFlow avoids the rollout mismatch of AR token prediction and promotes faster training convergence. Extensive experiments on benchmark datasets show that it consistently achieves superior forecasting performance with efficient inference.
\end{abstract}

\section{Introduction}
Multivariate time series (MTS) forecasting is an important research problem with broad applications in domains such as economics \cite{huang2024generative}, traffic systems \cite{shao2022decoupled, jiang2023pdformer}, energy management \cite{chen2024sdformer}, and weather prediction \cite{wu2020adversarial, wu2023interpretable}. In many real-world scenarios, forecasting involves a large number of variables evolving over time, making high-dimensional MTS forecasting particularly challenging. As a result, developing effective forecasting models for high-dimensional multivariate sequences has become an increasingly important topic.

Generative models provide a natural and effective framework for this setting, and have shown strong potential for scaling to high-dimensional MTS forecasting \cite{ sohl2015deep, dinh2016density, kingma2019introduction, song2020score}. Among them, diffusion-based methods have achieved competitive performance in recent years \cite{rasul2021autoregressive, shen2023non, meijer2024rise, cao2024timedit, wang2024timedart, ruan2025vision}. However, their inference typically requires many denoising steps, leading to substantial computational cost.

This limitation has motivated increasing interest in VAEs-based non-iterative forecasting frameworks, which perform generation in a learned latent space. Within this family, vector quantization (VQ) provides a controllable latent space modeling strategy by mapping complex multivariate sequences into a compact discrete codebook. Such vector-quantized representations have recently shown promise in token-based time series modeling and forecasting \cite{talukder2024totem, chen2024sdformer, feng2025hdt}. However, existing VQ-based forecasting methods typically rely on autoregressive (AR) token generation. Since AR models are usually trained with teacher forcing, they depend on ground-truth previous tokens during training but on their own predictions during inference, resulting in train--test mismatch and error accumulation during rollout \cite{schmidt2019generalization, bengio2015scheduled}.


Flow matching (FM) learns continuous transport and can
support non-autoregressive generation~\cite{liu2022flow}.
Recent time series applications include TSFlow, which
introduces Gaussian-process priors~\cite{kollovieh2024flow};
TFM, which learns history-conditioned Neural SDEs through
simulation-free training~\cite{zhang2024trajectory};
and FlowTime, which adopts teacher-forced autoregressive
forecasting~\cite{el2025probabilistic}.
These methods primarily model transport in observation
space, where high-dimensional MTS forecasting requires
capturing complex temporal and cross-variable dependencies
in the joint predictive distribution.
ProtoFlow addresses this challenge by jointly generating
future sequences in a compact vector-quantized latent
space. Its prototype prior incorporates learned latent
centers, empirical frequencies, and local variability
into the source distribution, providing a structured
initialization for conditional transport.
Empirically, this initialization supports earlier FM
training stabilization than its Gaussian-prior counterpart,
while enabling accurate few-step forecasting
in high-dimensional MTS settings.


Based on this observation, we propose \textbf{ProtoFlow},
a prototype-prior flow matching framework for
high-dimensional multivariate time series forecasting.
ProtoFlow combines vector-quantized representations
with a structured source prior derived from learned
codebook prototypes and their empirical statistics.
It jointly generates future latent sequences through
non-autoregressive flow matching, avoiding the
ground-truth versus generated-token conditioning
mismatch associated with teacher-forced autoregressive
rollout. Experiments on five real-world benchmarks
with several prediction horizons demonstrate
strong overall performance against competitive
diffusion- and flow-based baselines with few ODE
sampling steps. Further analyses show earlier FM
training stabilization and reduced sampling-step
sensitivity relative to Gaussian initialization,
supporting the effectiveness of reusing learned
latent structure for transport initialization.


\noindent\textbf{Main Contributions.}
\begin{itemize}
    \item We propose \textbf{ProtoFlow}, a framework for
    high-dimensional multivariate time series forecasting
    that combines vector-quantized representations with
    non-autoregressive flow matching to jointly generate
    future latent sequences, avoiding the conditioning
    mismatch associated with teacher-forced
    autoregressive rollout.

    \item We develop a structured source prior from
    learned prototype centers and empirical frequencies,
    extending the codebook's
    role from latent quantization to transport
    initialization. Compared with Gaussian
    initialization, this design empirically supports
    earlier FM training stabilization and reduced
    sensitivity to the sampling-step budget.

    \item Experiments on five real-world benchmarks
    with testing prediction horizons demonstrate
    strong overall probabilistic and deterministic
    forecasting performance against competitive
    diffusion- and flow-based baselines with
    substantially fewer sampling steps.
\end{itemize}

\section{Preliminaries}
We briefly review vector-quantized variational
autoencoders (VQ-VAE)~\cite{van2017neural} and
rectified flow matching, which underpin ProtoFlow.

\noindent\textbf{VQ-VAE Modeling.}
A VQ-VAE comprises an encoder $\mathcal{E}_{\psi}$,
a decoder $\mathcal{G}_{\phi}$, and a learned codebook
$\mathcal{Z}\in\mathbb{R}^{K\times n_z}$.
Given a time series
$\mathcal{X}_t\in\mathbb{R}^{\tau\times D}$,
the encoder produces
$\hat{z}=\mathcal{E}_{\psi}(\mathcal{X}_t)
\in\mathbb{R}^{L\times n_z}$,
where $L=\lfloor\tau/\rho\rfloor$ and $\rho$
is the temporal downsampling ratio.
Standard vector quantization assigns each latent
vector to its nearest codebook entry:
\begin{equation}
a_{\ell}
=
\arg\min_{k\in\{1,\dots,K\}}
\|\hat{z}^{\ell}-\mathcal{Z}_k\|_2^2,
\qquad
z_q^{\ell}=\mathcal{Z}_{a_{\ell}}.
\end{equation}
The discrete index sequence
$\mathbf{a}=(a_1,\dots,a_L)$ identifies the
quantized embeddings $z_q=q(\hat{z})$, from which
the decoder reconstructs
$\hat{\mathcal{X}}_t=\mathcal{G}_{\phi}(z_q)$.

Autoregressive token-based forecasting models
generate these indices conditioned on historical
information $\mathbf{c}$:
\begin{equation}
p(\mathbf{a}\mid\mathbf{c})
=
\prod_{\ell=1}^{L}
p(a_{\ell}\mid a_{<\ell},\mathbf{c}).
\end{equation}
Under teacher-forced training, conditioning uses
ground-truth preceding tokens, whereas inference
relies on generated tokens. This discrepancy can
propagate prediction errors during sequential
generation, motivating non-autoregressive
alternatives for latent forecasting.

\noindent\textbf{Rectified Flow Matching (RFM)
\cite{liu2022flow}.}
RFM learns a velocity field that transports samples
from a source distribution to a target distribution.
Given conditioning information $\bm{c}$, independently
sample $\bm{x}_0\sim\pi_0$ and
$\bm{x}_1\sim\pi_1(\cdot\mid\bm{c})$, and construct
\begin{equation}
\bm{x}_t=(1-t)\bm{x}_0+t\bm{x}_1,
\qquad t\sim\mathcal{U}(0,1).
\end{equation}
The network $\bm{v}_{\theta}$ is trained to predict
the displacement $\bm{x}_1-\bm{x}_0$:
\begin{equation}
\mathcal{L}_{\mathrm{RFM}}(\theta)
=
\mathbb{E}_{\bm{c},t,\bm{x}_0,\bm{x}_1}
\left[
\left\|
\bm{v}_{\theta}(\bm{x}_t,t\mid\bm{c})
-
(\bm{x}_1-\bm{x}_0)
\right\|_2^2
\right].
\label{eq:rfm_objective}
\end{equation}
At inference, a source sample is transformed by
numerically integrating the learned ODE:
\begin{equation}
\frac{\mathrm{d}\hat{\bm{x}}_t}{\mathrm{d}t}
=
\bm{v}_{\theta}(\hat{\bm{x}}_t,t\mid\bm{c}),
\qquad
\hat{\bm{x}}_0\sim\pi_0.
\label{eqn1}
\end{equation}
ProtoFlow applies this framework to vector-quantized
latent embeddings, using a structured source prior
derived from learned codebook prototypes and their
training statistics to initialize conditional transport.

\begin{figure*}[t!]
\vspace{-0.1in}
  \centering
   \includegraphics[width=0.95\linewidth]{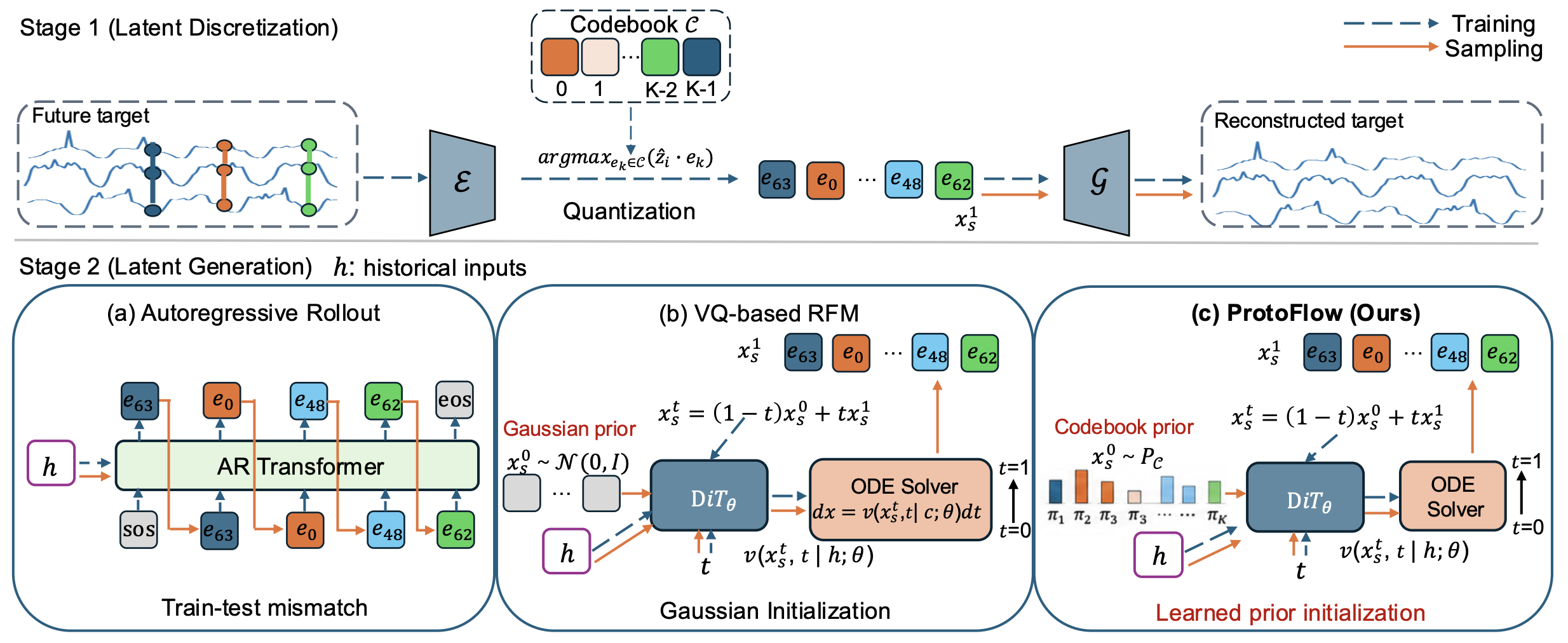}
   \vspace{-0.12in}
   \caption{Illustration of our ProtoFlow. Stage 1 discretizes the future target sequence into compact latent tokens via similarity-based VQ latent modeling. Stage 2 compares three latent-generation paradigms under the same tokenized space: (a) Autoregressive rollout, which suffers from training--inference mismatch; (b) Gaussian-initialized flow matching, which starts transport from an unstructured prior; and (c) The proposed ProtoFlow, which initializes conditional rectified flow from a learned codebook prior for latent flow matching generation. The future time features $\text{y}^{c}$ is omitted here.
 }
   \label{fig:overall}
\end{figure*}

\section{Method}

In this section, we present \textbf{ProtoFlow}, a prototype-prior flow-matching framework for multivariate time series forecasting. ProtoFlow first learns a vector-quantized tokenizer to map future sequences into a compact discrete latent space, and then performs conditional rectified flow matching initialized from a learned codebook prior. Thus, the codebook serves not only as a latent discretizer, but also as a structured prior for latent generation.

\subsection{Vector-Quantized Tokenizer}
\label{sec:VQ-stage1}

Stage 1 learns a vector-quantized tokenizer that maps
future target sequences into compact latent representations
for subsequent flow-based forecasting. Given a target
sequence $\mathcal{X}_t\in\mathbb{R}^{\tau\times D}$,
we use an encoder $\mathcal{E}_{\psi}$, a decoder
$\mathcal{G}_{\phi}$, and a learned codebook
$\mathcal{Z}\in\mathbb{R}^{K\times n_z}$ containing
$K$ embeddings of dimension $n_z$. The encoder produces
a latent sequence
$\hat{z}=\mathcal{E}_{\psi}(\mathcal{X}_t)
\in\mathbb{R}^{L\times n_z}$,
where $L=\lfloor\tau/\rho\rfloor$ and $\rho$ denotes
the temporal downsampling ratio.

Following SDformer~\cite{chen2024sdformer}, we adopt
similarity-driven vector quantization. Specifically,
we $\ell_2$-normalize the encoder outputs and codebook
embeddings, and assign each latent vector $\hat{z}^{\ell}$
to the code with the highest cosine similarity:
\begin{equation}
a_{\ell}
=
\underset{k\in\{1,\dots,K\}}{\arg\max}
\frac{
\langle\hat{z}^{\ell},\mathcal{Z}_k\rangle
}{
\|\hat{z}^{\ell}\|_2\|\mathcal{Z}_k\|_2
},
\qquad
z_q^{\ell}=\mathcal{Z}_{a_{\ell}}.
\label{eq:similarity_quantization}
\end{equation}
For unit-normalized vectors, this criterion reduces
to maximizing their inner product, making code
assignment depend on latent direction rather than
magnitude. The resulting quantized embedding sequence is
\begin{equation}
\mathcal{X}_s
=
[z_q^1,\dots,z_q^L]
=
q(\mathcal{E}_{\psi}(\mathcal{X}_t)),
\label{eqn:compression-ratio}
\end{equation}
which provides the latent target representation for
Stage 2.

The decoder reconstructs the target sequence as
$\hat{\mathcal{X}}_t=\mathcal{G}_{\phi}(\mathcal{X}_s)$.
We train the tokenizer using a reconstruction loss
and a cosine-based commitment loss:
\begin{equation}
\mathcal{L}_{\mathrm{VQ}}
=
\|\mathcal{X}_t-\hat{\mathcal{X}}_t\|_2^2
+
\frac{\beta}{L}
\sum_{\ell=1}^{L}
\left(
1-
\left\langle
\hat{z}^{\ell},
\operatorname{sg}[z_q^{\ell}]
\right\rangle
\right),
\label{eq:vq_loss_main}
\end{equation}
where $z_q=\mathcal{X}_s$,
$\hat{z}^{\ell}$ and $z_q^{\ell}$ are
$\ell_2$-normalized encoder and codebook embeddings,
respectively, and $\operatorname{sg}[\cdot]$
denotes the stop-gradient operator.
The coefficient $\beta$ controls the commitment
strength, encouraging encoder outputs to align
with their assigned codebook embeddings. We update the codebook using exponential moving averages
and reset inactive entries to stabilize codebook learning, which is shown in Appendix Eqn~\ref{eq:reset}.

\subsection{Codebook-Prior Rectified Flow Matching}
\label{sec:proto_rfm}

After learning the Stage 1 tokenizer, we model the
conditional distribution of future latent sequences
using rectified flow matching. The target embedding
sequence is
$\mathbf{x}_s^1=\mathcal{X}_s\in\mathbb{R}^{L\times n_z}$,
where $L=\lfloor\tau/\rho\rfloor$ and $\rho$ denotes
the temporal downsampling ratio. The superscripts
$0$ and $1$ denote the source and target endpoints
of latent transport, respectively.

\noindent\textbf{Codebook-prior Initialization.}
A key design of ProtoFlow is to initialize latent
transport from a structured codebook prior.
We construct this prior from the learned codebook
$\mathcal{Z}$ and training latent statistics.
Each component is specified by a mixing weight
$\pi_k$, a prototype center
$\mu_k\in\mathbb{R}^{n_z}$ associated with the
$k$-th codebook entry, and local perturbation
parameters $(B_k,\sigma_k)$.
Here, $B_k\in\mathbb{R}^{n_z\times r}$ contains
the principal local basis directions,
$\sigma_k\in\mathbb{R}_{\geq 0}^{r}$ contains
their scales, and $r$ denotes the local covariance
rank. The single-position source prior is
\begin{equation}
p_{\mathcal{Z}}(z)
=
\sum_{k=1}^{K}
\pi_k\,p_k(z;\mu_k,B_k,\sigma_k,\gamma),
\qquad
\pi_k\geq 0,\quad
\sum_{k=1}^{K}\pi_k=1,
\label{eq:proto_prior}
\end{equation}
where $\gamma\geq 0$ controls the perturbation
magnitude, and each component is defined by the
normalized sampling procedure below.

For each latent position $\ell$, we independently
sample a component index and a local perturbation:
\begin{equation}
s_\ell^0\sim\mathrm{Cat}(\pi_1,\dots,\pi_K),
\qquad
\varepsilon_\ell\sim\mathcal{N}(0,I_r),
\end{equation}
and construct
\begin{equation}
\mathbf{x}_{s,\ell}^0
=
\mathrm{Norm}\!\left(
\mu_{s_\ell^0}
+
\gamma B_{s_\ell^0}
\bigl(\sigma_{s_\ell^0}\odot\varepsilon_\ell\bigr)
\right),
\label{eq:proto_source_position}
\end{equation}
where $\mathrm{Norm}(\cdot)$ denotes $\ell_2$
normalization and $\odot$ denotes elementwise
multiplication. The parameters $(B_k,\sigma_k)$
describe the local perturbation before normalization. Stacking all positions yields the source latent sequence
\begin{equation}
\mathbf{x}_s^0 = [x_{s,1}^0,x_{s,2}^0,\dots,x_{s,L}^0].
\label{eq:proto_source}
\end{equation}
This initialization is therefore centered on learned latent prototypes and their local neighborhoods, rather than on a generic isotropic distribution.

\noindent\textbf{Conditional Rectified Flow Matching.}
Let $\mathbf{c}=(\mathcal{H}_p,\mathbf{y}^{c})$
denote the forecasting conditions, comprising the
encoded historical context and known temporal
features for the prediction window.
For a target--condition pair
$(\mathbf{x}_s^1,\mathbf{c})$ and an independently
sampled source $\mathbf{x}_s^0\sim P_{\mathcal{Z}}$,
we define the linear interpolation
\begin{equation}
\mathbf{x}_s^t
=
(1-t)\mathbf{x}_s^0+t\mathbf{x}_s^1,
\qquad t\in[0,1],
\label{eq:proto_xt}
\end{equation}
where $t$ denotes flow time. Its target velocity is
\begin{equation}
\mathbf{u}_s^t
=
\frac{\mathrm{d}\mathbf{x}_s^t}{\mathrm{d}t}
=
\mathbf{x}_s^1-\mathbf{x}_s^0.
\label{eq:proto_ut}
\end{equation}

We parameterize the vector field as
$\mathbf{v}_{\theta}(\mathbf{x}_s^t,t\mid\mathbf{c})
=
\mathrm{DiT}_{\theta}(\mathbf{x}_s^t,t,\mathbf{c})$
and minimize
\begin{equation}
\mathcal{L}_{\mathrm{FM}}
=
\mathbb{E}_{\substack{
(\mathbf{x}_s^1,\mathbf{c})\sim q_{\mathrm{data}},\\
\mathbf{x}_s^0\sim P_{\mathcal{Z}},\;t\sim q_t
}}
\left[
\left\|
\mathbf{v}_{\theta}(\mathbf{x}_s^t,t\mid\mathbf{c})
-
(\mathbf{x}_s^1-\mathbf{x}_s^0)
\right\|_F^2
\right],
\label{eq:proto_loss}
\end{equation}
where $q_{\mathrm{data}}$ denotes the empirical
distribution of target latent sequences and their
conditions, $q_t$ is the training-time sampling
distribution on $[0,1]$, and $\|\cdot\|_F$ denotes
the Frobenius norm.

\noindent\textbf{Sampling.}
At inference, we sample
$\hat{\mathbf{x}}_s^{t_0}\sim P_{\mathcal{Z}}$
and integrate the learned vector field over a grid
$0=t_0<t_1<\cdots<t_M=1$ using Euler updates:
\begin{equation}
\hat{\mathbf{x}}_s^{t_{j+1}}
=
\hat{\mathbf{x}}_s^{t_j}
+
\Delta t_j\,
\mathbf{v}_{\theta}
\bigl(\hat{\mathbf{x}}_s^{t_j},t_j\mid\mathbf{c}\bigr),
\qquad
\Delta t_j=t_{j+1}-t_j,
\label{eq:proto_euler}
\end{equation}
for $j=0,\dots,M-1$.
Each update jointly evolves the entire future
latent sequence, enabling non-autoregressive
forecast generation. The final latent sample is
decoded into the forecasting target:
\begin{equation}
\hat{\mathcal{X}}_t
=
\mathcal{G}_{\phi}(\hat{\mathbf{x}}_s^1).
\label{eq:proto_decode}
\end{equation}

\noindent\textbf{Architecture.}
We use a DiT-based vector field~\cite{peebles2023scalable}
conditioned on flow time, historical context $\mathcal{H}_p$,
and future temporal features $\mathbf{y}^{c}$.
The context encoder combines convolutions and
Transformer-XL~\cite{dai2019transformer} with relative
positional embeddings, while flow-time embeddings
control adaptive LayerNorm.
Figures~\ref{fig:overall} and~\ref{fig:details}
illustrate the architecture.

\subsection{Theoretical Motivation for Codebook-Prior Initialization}
\label{sec:proto_theory}

The benefit of codebook-prior initialization can be understood from the geometry of rectified transport in latent space. Let $\mathcal P_1$ denote the target latent distribution of $\mathbf{x}_s^1$, and let $\mathcal P_0$ denote an arbitrary source prior over the same latent space $\mathbb{R}^{L\times n_z}$. For any coupling $\Gamma\in\Pi(\mathcal P_0,\mathcal P_1)$, consider the rectified interpolation
\begin{equation}
\mathbf{x}_s^t=(1-t)\mathbf{x}_s^0+t\mathbf{x}_s^1,\qquad t\in[0,1],
\label{eq:main_rectified_interp}
\end{equation}
where $(\mathbf{x}_s^0,\mathbf{x}_s^1)\sim\Gamma$. We define the expected squared rectified transport cost under $\Gamma$ as
\begin{equation}
\mathcal C_{\mathrm{rect}}(\Gamma)
:=
\mathbb E_{(\mathbf{x}_s^0,\mathbf{x}_s^1)\sim\Gamma}
\big[
\|\mathbf{x}_s^1-\mathbf{x}_s^0\|_F^2
\big].
\label{eq:main_rect_cost}
\end{equation}

\begin{proposition}
\label{prop:main_w2}
Assume $\mathcal P_0$ and $\mathcal P_1$ have finite second moments. Then the minimum expected squared rectified transport cost from $\mathcal P_0$ to $\mathcal P_1$ satisfies
\begin{equation}
\inf_{\Gamma\in\Pi(\mathcal P_0,\mathcal P_1)}
\mathcal C_{\mathrm{rect}}(\Gamma)
=
W_2^2(\mathcal P_0,\mathcal P_1),
\label{eq:main_w2_result}
\end{equation}
where $W_2(\mathcal P_0,\mathcal P_1)$ is the $2$-Wasserstein distance under the Frobenius ground cost.
\end{proposition}


Proposition~\ref{prop:main_w2} characterizes the minimum
expected squared transport cost over all source--target
couplings, providing a geometric motivation for source-prior
design. ProtoFlow uses independent source--target sampling;
the benefits of its prototype prior are evaluated empirically of Figure~\ref{fig:sampling}
. Let $\mathcal P_C:=P_{\mathcal Z}$ denote ProtoFlow's
sequence-level codebook prior, and let $\mathcal P_G=\mathcal N(0,I_{Ln_z})$ denote Gaussian initialization in the same latent space.
If
\begin{equation}
W_2(\mathcal P_C,\mathcal P_1)
<
W_2(\mathcal P_G,\mathcal P_1),
\label{eq:main_closer_prior}
\end{equation}
then the codebook prior induces a smaller optimal rectified transport cost than Gaussian initialization. Under this assumption, this conditional result provides a geometric motivation for incorporating learned latent structure into the
source prior. Detailed derivation is given in Appendix~\ref{app:proto_theory}. \\
This geometric motivation is complemented by
Figure~\ref{fig:sampling}, where ProtoFlow achieves
lower empirical Wasserstein distances between generated
forecasts and observed targets, together with strong
few-step forecasting performance. These results provide
empirical support for the practical benefits of
prototype-informed initialization.

\section{Experiments}
\textbf{Benchmarks.} Our evaluation uses five widely adopted real-world benchmarks
for high-dimensional MTS forecasting:
Solar and Electricity \cite{lai2018modeling},
Traffic and Taxi \cite{salinas2019high}, and
Wikipedia \cite{gasthaus2019probabilistic}.
These datasets span energy, transportation, and web-traffic
applications, with sampling intervals ranging from
30 minutes to one day. Additional dataset details are
provided in Appendix~\ref{supp:more_dataset_details}.

\textbf{Baselines.} We compare ProtoFlow with strong generative forecasting
baselines from the diffusion and flow families. For diffusion-based forecasting, we include TimeGrad~\cite{rasul2021autoregressive}, CSDI~\cite{tashiro2021csdi}, TSDiff~\cite{kollovieh2023predict}, MG-TSD~\cite{fan2024mg}, DyDiff~\cite{guo2025dynamical}, and NsDiff~\cite{ye2025non}. We additionally consider the recent popular flow-based approachs FlowTime~\cite{el2025probabilistic} and TSFlow~\cite{kollovieh2025flow}. Detailed baseline descriptions are deferred to Appendix~\ref{sec:related_work}.


\textbf{Evaluation Metrics.} We evaluate both probabilistic and deterministic forecasting quality. For probabilistic forecasting, we report the continuous ranked probability score on the aggregated series, $\mathrm{CRPS}_{\mathrm{sum}}$ \cite{matheson1976scoring}. For deterministic accuracy, we report $\mathrm{NRMSE}_{\mathrm{sum}}$, which provides a scale-normalized comparison across datasets with different value ranges. Following standard probabilistic forecasting protocols, all results are computed from 100 generated samples.More details on the evaluation metrics are given in Appendix~\ref{supp:more_impl_details}. 

\noindent\textbf{Implementation Details.} We use a history length of 96 and prediction lengths
of $\{48,96\}$, following high-dimensional time series forecasting settings \cite{rasul2021autoregressive, tashiro2021csdi}. RevIN~\cite{kim2021reversible} is
applied except Solar.
All experiments run on a single NVIDIA A100 GPU
with 40\,GB memory, with results averaged over
three runs. More details are shown in
Appendix Tables~\ref{tab:ts_data} and ~\ref{tab:ts_data2}.


\subsection{Main Results}  


We evaluate ProtoFlow on five widely used high-dimensional
MTS benchmarks with prediction horizons of 48 and 96 steps.
As shown in Table~\ref{tab:appendix_resuls-details}, ProtoFlow
achieves strong overall performance in both probabilistic
forecasting ($\text{CRPS}_{\text{sum}}$) and point
forecasting ($\text{NRMSE}_{\text{sum}}$), obtaining
the best or joint-best results in 18 out of 20 comparisons,
including comparisons with the Gaussian-prior variant
ProtoFlow-G. Specifically,

\begin{table*}[!h]
    \vspace{-0.1in}
    \begin{center}
    \begin{threeparttable}
    \large
    \resizebox{0.95\linewidth}{!}{
    \begin{tabular}{l|cc|cc|cc|cc|cc}
    \multicolumn{1}{c}{} & \multicolumn{2}{c}{Electricity} & \multicolumn{2}{c}{Solar} & \multicolumn{2}{c}{Traffic} & \multicolumn{2}{c}{Taxi} & \multicolumn{2}{c}{Wikipedia}\\
    Horizon+Method & $\text{CRPS}_\text{sum}$ & $\text{NRMSE}_\text{sum}$  & $\text{CRPS}_\text{sum}$ & $\text{NRMSE}_\text{sum}$ &  $\text{CRPS}_\text{sum}$ & $\text{NRMSE}_\text{sum}$ &  $\text{CRPS}_\text{sum}$ & $\text{NRMSE}_\text{sum}$ &  $\text{CRPS}_\text{sum}$ & $\text{NRMSE}_\text{sum}$ \\
    \midrule
   \textbf{48 Steps} & &  &  &  &  &  &  & \\
    ~~$\bullet~$TimeGrad~\citeyearpar{rasul2021autoregressive} & 0.031 & 0.044 & 0.324 & 0.736 & 0.043 & 0.072 & 0.252 & 0.334 & 0.082 & 0.105  \\
    ~~$\bullet~$CSDI~\citeyearpar{tashiro2021csdi} & 0.029 & 0.050 & 0.328 & 0.742 & - & - & - & - & -& -  \\
    ~~$\bullet~$TSDiff~\citeyearpar{kollovieh2023predict} & 0.036 & 0.057 & 0.324 & 0.731 & 0.057 & 0.070 & 0.243 & 0.330 & 0.079 & \bf 0.091  \\
    ~~$\bullet~$MG-TSD~\citeyearpar{fan2024mg} & 0.024 & 0.038 & 0.328 & 0.737 & 0.042 & 0.067 & 0.217 & 0.327 & 0.071 & 0.096 \\
    ~~$\bullet~$FlowTime~\citeyearpar{el2025probabilistic} & 0.033 & 0.052 & 0.321 & 0.729 & 0.054 & 0.077 & 0.419 & 0.514 & 0.091 & 0.115  \\
    ~~$\bullet~$DyDiff~\citeyearpar{guo2025dynamical} & 0.027 & 0.043 & 0.323 & 0.724 & 0.043 & 0.064 & 0.196 & 0.313 & 0.070 & 0.097  \\
    ~~$\bullet~$NsDiff~\citeyearpar{ye2025non} & 0.025 & 0.040 & 0.311 & 0.703 & \underline{0.040} & \bf 0.061 & \underline{0.182} & 0.299 & 0.074 & \underline{0.094}  \\
    ~~$\bullet~$TSFlow~\citeyearpar{kollovieh2025flow} & \underline{0.015} & \underline{0.036} & \underline{0.289} & \underline{0.681} & \bf 0.038 & \underline{0.064} & 0.184 & \underline{0.294} & \underline{0.065} & 0.107  \\
    ~~$\bullet~$ProtoFlow-G &  \textcolor{gray}{0.014} &  \textcolor{gray}{0.031} & \textcolor{gray}{0.301}  & \textcolor{gray}{0.689} & \textcolor{gray}{0.035} & \textcolor{gray}{0.057}  & \textcolor{gray}{0.153}  & \textcolor{gray}{0.279}  & \textcolor{gray}{0.067} & \textcolor{gray}{0.118} \\
    \rowcolor{black!8}~~$\bullet~$\textbf{ProtoFlow (ours)} & \bf 0.013 & \bf 0.028 & \bf 0.275 & \bf 0.626 & 0.041 & 0.066 & \bf 0.150 & \bf 0.272 & \bf 0.063 & \bf 0.091 \\
    \cmidrule{1-11}
   \textbf{96 Steps} & &  &  &  &  &  &  & \\
    ~~$\bullet~$TimeGrad~\citeyearpar{rasul2021autoregressive} & 0.039 & 0.047 & 0.347 & 0.749 & 0.046 & 0.079 & 0.369 & 0.454 & 0.109 & 0.138  \\
    ~~$\bullet~$CSDI~\citeyearpar{tashiro2021csdi} & 0.047 & 0.044 & 0.330 & 0.753 & - & - & - & - & -& -  \\
    ~~$\bullet~$TSDiff~\citeyearpar{kollovieh2023predict} & 0.039 & 0.048 & 0.336 & 0.748 & 0.063 & 0.076 & 0.338 & 0.419 & 0.113 & 0.135  \\
    ~~$\bullet~$MG-TSD~\citeyearpar{fan2024mg} & 0.043 & 0.044 & 0.349 & 0.762 & 0.047 & 0.072 & 0.274 & 0.364 & \underline{0.094} & 0.126 \\
    ~~$\bullet~$FlowTime~\citeyearpar{el2025probabilistic} & 0.048 & 0.059 & 0.328 & 0.751 & 0.069 & 0.083 & 0.474 & 0.576 & 0.135 & 0.163  \\
    ~~$\bullet~$DyDiff~\citeyearpar{guo2025dynamical} & 0.036 & 0.052 & 0.334 & 0.745 & 0.056 & 0.079 & 0.247 & 0.324 & 0.101 & 0.109  \\
    ~~$\bullet~$NsDiff~\citeyearpar{ye2025non} & 0.033 & 0.047 & 0.309 & 0.714 & 0.047 & 0.069 & 0.236 & \underline{0.312} & 0.099 & \underline{0.107}  \\
    ~~$\bullet~$TSFlow~\citeyearpar{kollovieh2025flow} & \underline{0.025} & \underline{0.044} & \underline{0.299} & 
    \underline{0.703} & \underline{0.044} & \underline{0.068} & \underline{0.212} & 0.346 & 0.103 & 0.136  \\
    ~~$\bullet~$ProtoFlow-G & \textcolor{gray}{0.019} & \textcolor{gray}{0.037}  & \textcolor{gray}{0.305}  & 
    \textcolor{gray}{0.717}  & \textcolor{gray}{0.036} & \textcolor{gray}{0.058} & \textcolor{gray}{0.159}  & \textcolor{gray}{0.287} & \textcolor{gray}{0.096}  & \textcolor{gray}{0.114} \\
    \rowcolor{black!8}~~$\bullet~$\textbf{ProtoFlow (ours)} & \bf 0.015 & \bf 0.022 & \bf 0.284 & \bf 0.638 & \bf 0.034 & \bf 0.054 & \bf 0.157 & \bf 0.281 & \bf 0.092 & \bf 0.101 \\
    \end{tabular}
    }  
    \caption{The detailed results of ProtoFlow  on the 48 and 96 prediction horizons. ProtoFlow-G: the variant of ProtoFlow that uses a standard Gaussian prior. - marks out-of-memory failures, Best and second-best forecasting results are bolded and underlined, respectively. }
    \label{tab:appendix_resuls-details}
    \end{threeparttable}
    \end{center}
\end{table*}
\textbf{1)} Compared with flow-based baselines, ProtoFlow reduces
$\text{CRPS}_{\text{sum}}$ and
$\text{NRMSE}_{\text{sum}}$ by an average of
13.6\% and 17.4\% over TSFlow, and by 42.6\% and 34.4\%
over FlowTime, respectively. Compared with NsDiff,
a strong diffusion-based baseline, ProtoFlow achieves
average reductions of 22.0\% and 14.6\%, respectively.
All percentages are calculated by averaging the relative
error reductions across the five datasets and two
prediction horizons. These results demonstrate the
effectiveness of ProtoFlow across diverse high-dimensional
forecasting settings. \textbf{2)} ProtoFlow achieves strong sampling efficiency compared with competitive flow-based and diffusion-based methods. Specifically, it uses only 3 ODE sampling steps on all datasets except Wikipedia, which requires 5 steps. In comparison, TSFlow, a strong flow-based baseline, and NsDiff, a strong diffusion-based baseline, use 16 and 20 sampling steps, respectively, following the settings in their original papers. ProtoFlow achieves better overall forecasting accuracy than both methods with substantially fewer sampling steps, demonstrating a favorable balance between prediction quality and sampling efficiency. \textbf{3)}
To examine the contribution of the prototype prior, we compare ProtoFlow with ProtoFlow-G, which uses a standard Gaussian prior during both training and inference. However, ProtoFlow converges in fewer training iterations (Figure~\ref{fig:efficiency} (a)), supporting the benefit of the prototype prior in accelerating flow matching training convergence.

\subsection{Effects of latent representation and generation strategy.} 
We analyze the effects of latent representation and generation strategy in Table~\ref{tab:representation_results_summary}.

\textbf{1)} VQ-based variants outperform their VAE-based counterparts on both metrics across all five datasets under both AR and FM generation. Specifically, replacing VAE with VQ reduces the dataset-averaged $\text{CRPS}_{\text{sum}}$ and $\text{NRMSE}_{\text{sum}}$ by 75.0\% and 73.3\% under AR generation, and by 57.3\% and 51.9\% under Gaussian-prior FM, respectively. These consistent improvements support the effectiveness of vector-quantized representations across different generation strategies. \textbf{2)} With VQ representations, ProtoFlow-G consistently outperforms VQ + AR across all five datasets, reducing the dataset-averaged  $\text{CRPS}_{\text{sum}}$ and $\text{NRMSE}_{\text{sum}}$ by 19.4\% and 14.2\%, respectively. By modeling the future latent sequence jointly, non-autoregressive FM avoids the discrepancy between conditioning on ground-truth preceding tokens during teacher-forced training and on generated tokens during AR inference. These results support the advantage of non-autoregressive FM in this setting. \textbf{3)} Overall, these results support the combination of
vector-quantized representations and non-autoregressive
flow matching. ProtoFlow further improves upon ProtoFlow-G
on both metrics across all five datasets. Together with the
earlier stabilization of its normalized training loss
in Figure~\ref{fig:efficiency}(a), these results support
the benefit of incorporating learned prototype information
into the FM source distribution.

\begin{table*}[!h]
  \centering
  
\resizebox{0.95\textwidth}{!}{
\begin{tabular}{lcccccccccc|cc}
\toprule
\multirow{2}[2]{*}{\centering Method} &  \multicolumn{2}{c}{\textbf{Electricity}} & \multicolumn{2}{c}{\textbf{Solar}} &  \multicolumn{2}{c}{\textbf{Traffic}} & \multicolumn{2}{c}{\textbf{Taxi}} & \multicolumn{2}{c}{\textbf{Wikipedia}} & \multicolumn{2}{c}{\textbf{Avg.}} \\
\cmidrule(lr){2-3}  \cmidrule(lr){4-5}  \cmidrule(lr){6-7}  \cmidrule(lr){8-9}  \cmidrule(lr){10-11} \cmidrule(lr){12-13}  
& C-S & NM-S & C-S & NM-S & C-S & NM-S & C-S & NM-S & C-S & NM-S & C-S & NM-S\\
\midrule


VAE + AR & 0.421 & 0.544  & 1.154  & 2.618  & 1.005 & 1.511 & 0.337 & 0.448 & 0.130 & 0.169 & 0.609 & 1.058\\

VAE + FM & 0.129 & 0.213  & 0.643 & 1.324  & 0.323 & 0.564 & 0.227 & 0.293 & 0.117 & 0.128 &0.288 & 0.504\\

VQ + AR & 0.037 & 0.054  & 0.347  & 0.833 & 0.066 & 0.084 & 0.205 & 0.322 & 0.108 & 0.121 & 0.153 & 0.283 \\

ProtoFlow-G & 0.019 & 0.037  & 0.305  &  0.717  & 0.036 & 0.058 & 0.159 & 0.287 & 0.096 & 0.114 & \underline{0.123} & \underline{0.243}   \\






ProtoFlow & \bf 0.015 & \bf 0.022 & \bf 0.284 & \bf 0.638 & \bf 0.034 & \bf 0.054 & \bf 0.157 & \bf 0.281 & \bf 0.092 & \bf 0.101 & \bf 0.116 & \bf 0.219\\


\bottomrule

\end{tabular}%
}
\vspace{-0.1in}
\caption{
  The comparison results of 96 prediction horizons with baselines regarding $\text{CRPS}_{\text{sum}}$ (C-S) and $\text{NRMSE}_{\text{sum}}$ (NM-S). AR: Auto-regressive Transformer and FM: Flow Matching .
}
\label{tab:representation_results_summary}%
\vskip -0.1in
\end{table*}%

\vspace{-0.08in}
\subsection{Effects of the prototype prior on FM training/inference efficiency.} We smooth training losses with a five-epoch moving average and normalize each curve as $\widetilde{L}_e =
 (\bar{L}_e - L_{\mathrm{end}}) /
 (\bar{L}_1 - L_{\mathrm{end}})$, where $\bar{L}_e$ denotes the smoothed loss at epoch $e$, and $L_{\mathrm{end}}$ is the median of the smoothed losses over the final $5\%$ of training epochs for the corresponding run. Figure~\ref{fig:efficiency}(a) shows that ProtoFlow approaches its own terminal loss level in fewer epochs than ProtoFlow-G and TSFlow across all three datasets. The comparison with ProtoFlow-G, which shares the same representation and generation architecture, supports earlier FM training stabilization with prototype initialization. These curves describe relative loss reduction toward each run's observed endpoint and do not establish equivalent absolute optimization errors or predictive accuracy across methods.

Figure~\ref{fig:efficiency} (b) reports inference speedups of $3.66$--$3.90\times$ over TSFlow. ProtoFlow combines 3-step ODE sampling with non-autoregressive latent generation, avoiding autoregressive dependencies, whereas TSFlow uses 16 steps following its original configuration. These latency gains accompany a larger training-memory footprint ($4.08$--$7.49\times$),associated with the additional components used for latent representation learning and generation. Meanwhile, ProtoFlow outperforms TSFlow on both forecasting metrics across all three datasets and generally improves upon ProtoFlow-G (Figures~\ref{fig:efficiency} (c)--(d)).

\begin{figure*}[t!]
\vspace{-0.1in}
  \centering
   \includegraphics[width=1.0\linewidth]{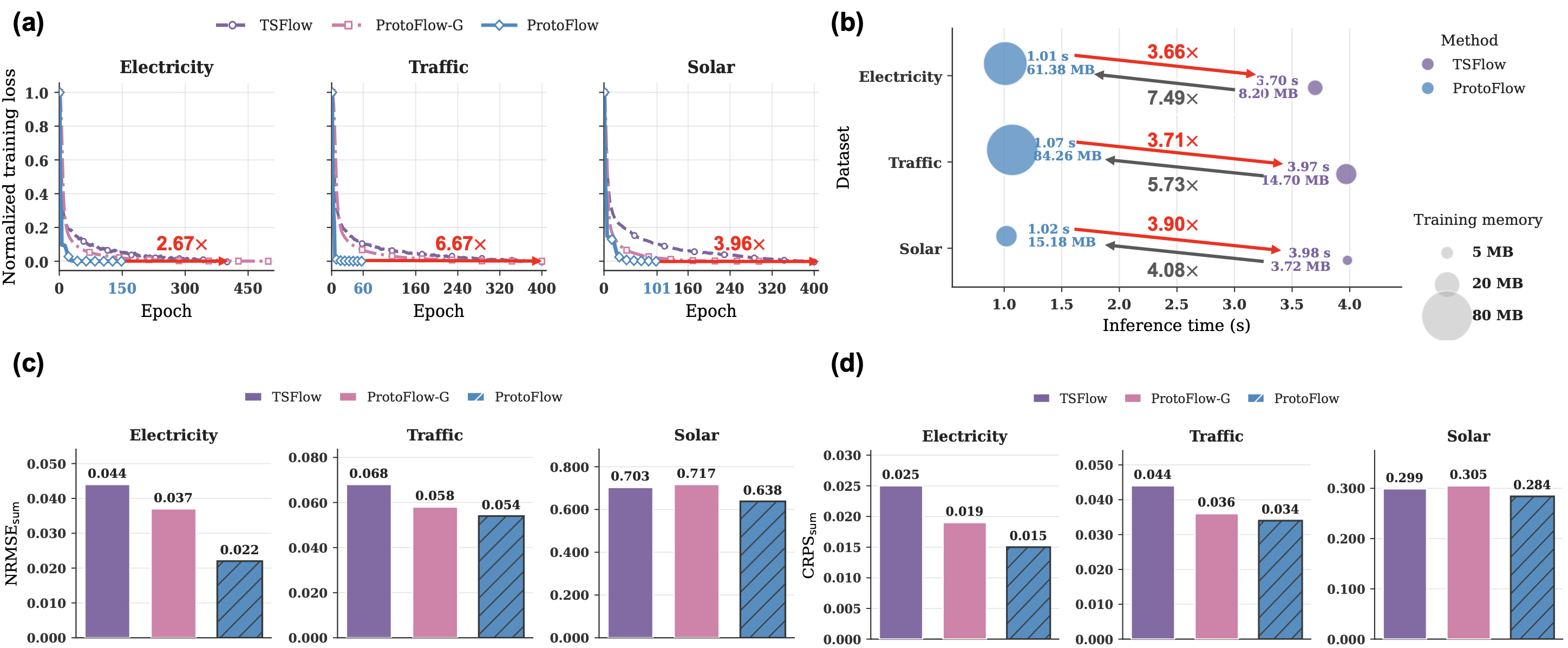}
   \vspace{-0.16in}
   \caption{Comparison of training convergence (\textbf{a}), inference efficiency, 
training memory (\textbf{b}), and forecasting performance (\textbf{c-d}) on Electricity, Traffic,
and Solar with a prediction length of 96.}
   \label{fig:efficiency}
\end{figure*}

\begin{table*}[h]
  \centering
\resizebox{0.95\textwidth}{!}{
\begin{tabular}{lcccccccccc|cc}
\toprule
\multirow{2}[2]{*}{\centering Method} &  \multicolumn{2}{c}{\textbf{Electricity}} & \multicolumn{2}{c}{\textbf{Solar}} &  \multicolumn{2}{c}{\textbf{Traffic}} & \multicolumn{2}{c}{\textbf{Taxi}} & \multicolumn{2}{c}{\textbf{Wikipedia}} & \multicolumn{2}{c}{\textbf{Avg.}} \\
\cmidrule(lr){2-3}  \cmidrule(lr){4-5}  \cmidrule(lr){6-7}  \cmidrule(lr){8-9}  \cmidrule(lr){10-11} \cmidrule(lr){12-13}  
& C-S & NM-S & C-S & NM-S & C-S & NM-S & C-S & NM-S & C-S & NM-S & C-S & NM-S\\
\midrule

ProtoFlow-G & 0.019 & 0.037  & \underline{0.305}  &  \underline{0.717}  &  \underline{0.036} & 0.058 & 0.159 & 0.287 & \underline{0.096} & 0.114 & \underline{0.123} & \underline{0.243}   \\

ProtoFlow-uniform  & \underline{0.018} & \underline{0.030}  & 0.356 & 0.769 & \underline{0.036} & \underline{0.056} & \bf 0.157 & \underline{0.284} & 0.099 & \underline{0.108} & 0.133 & 0.249 \\

ProtoFlow & \bf 0.015 & \bf 0.022 & \bf 0.284 & \bf 0.638 & \bf 0.034 & \bf 0.054 & \bf 0.157 & \bf 0.281 & \bf 0.092 & \bf 0.101 & \bf 0.116 & \bf 0.219\\

\bottomrule
\end{tabular}%
}
\vspace{-0.1in}
\caption{The forecasting performance with frequency-weighted
prototype (ProtoFlow), uniform-weight prototype (ProtoFlow-uniform), and isotropic Gaussian
priors (ProtoFlow-G) at a prediction length of 96.
}
\label{tab:prototype_frequency}%
\vskip -0.1in
\end{table*}%


\subsection{Effects of codebook size and prototype frequency information.} The prototype prior provides a data-informed initialization
for flow matching through learned latent structure.
We examine codebook size, which shapes latent quantization
and the mixture prior, and prototype frequency information,
which determines probability allocation across prototypes.
Specifically, we vary $K\in\{64,128,256,512\}$ with local
covariance rank $r=2$ fixed across datasets, and compare
frequency-based mixing weights with uniform weights.
Figure~\ref{fig:codebook_size}, Table~\ref{tab:prototype_frequency} and~\ref{tab:results_codereset_main} report their respective
effects on forecasting performance.

\begin{figure*}[!h]
\vspace{-0.1in}
  \centering
   \includegraphics[width=0.95\linewidth]{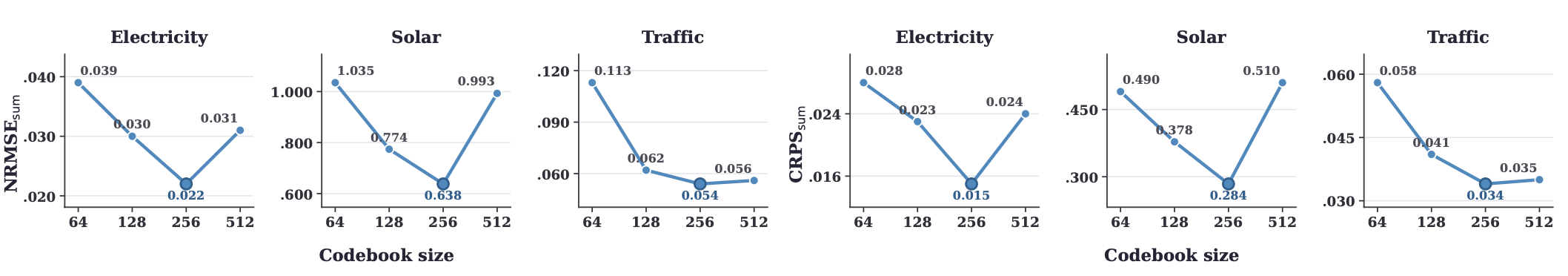}
   \vspace{-0.12in}
   \caption{Effects of codebook size on ProtoFlow at the 96-step prediction length.}
   \label{fig:codebook_size}
\end{figure*}



Figure~\ref{fig:codebook_size} shows that both metrics
improve from $K=64$ to $256$, which yields the best
observed results on all three datasets. At $K=512$,
errors rise substantially on Electricity and Solar
but only slightly on Traffic. Larger codebooks may
leave fewer samples per prototype for estimating local
prior statistics, suggesting a balance between
representational capacity and reliable prior estimation. Then, we replace the mixing weights estimated from training
statistics with $\pi_k=1/K$ in Eq.~\eqref{eq:proto_prior},
yielding
$p_{\mathcal{C}}^{\mathrm{uniform}}(\mathbf{x})
=\frac{1}{K}\sum_{k=1}^{K}
p_k(\mathbf{x};\mu_k,B_k,\sigma_k)$. Table~\ref{tab:prototype_frequency} compares ProtoFlow-G,
ProtoFlow-uniform, and ProtoFlow using respective ODE step
counts of $\{10,50,10,20,10\}$, $\{5,50,5,10,10\}$, and
$\{3,3,3,3,5\}$, following the dataset order in the table.
ProtoFlow achieves the best or joint-best performance
on both metrics across all five datasets while using
fewer steps. Although uniform weighting remains competitive
in several settings, its larger sampling budgets offer
no overall accuracy advantage. These findings support
combining learned prototype structure with empirical
frequency information for accurate few-step forecasting.

\begin{table}[htp]
  \centering
\resizebox{0.90\linewidth}{!}{
\begin{tabular}{lccccccccc}
\toprule
\multirow{2}[2]{*}{} &  \multicolumn{3}{c}{\textbf{Electricity}} & \multicolumn{3}{c}{\textbf{Solar}} &  \multicolumn{3}{c}{\textbf{Traffic}} \\
\cmidrule(lr){2-4}  \cmidrule(lr){5-7}  \cmidrule(lr){8-10}  
& Frequency & Uniform & Stage1 MSE & Frequency & Uniform & Stage1 MSE & Frequency & Uniform & Stage1 MSE \\
\midrule
{\bf Coodebook size} &  &   &  &  &  &  \\
\hspace{0.5em} 64
& 0.039  & 0.044 &\cellcolor{red!15}7.76e-4 & 1.035 & 0.962 & \cellcolor{red!15}7.70e-4 & 0.113 & 0.105 & \cellcolor{red!15}5.82e-4 \\
\hspace{0.5em} 128
& 0.030 & 0.032 & 7.13e-4 & 0.774 & 0.761 & 6.75e-4 & 0.062 & 0.066 & 5.23e-4\\
\hspace{0.5em} 256
& 0.022 & 0.030 & \cellcolor{yellow!15}6.02e-4 & 0.638 & 0.769 &  \cellcolor{yellow!15}6.13e-4 & 0.054 & 0.056 & \cellcolor{yellow!15}4.78e-4\\
\hspace{0.5em} 512
& 0.031 & 0.033 & \cellcolor{yellow!15}6.01e-4 & 0.993 & 0.758 &  \cellcolor{yellow!15}6.18e-4 & 0.056 & 0.057 & \cellcolor{yellow!15}4.69e-4\\

\bottomrule
\end{tabular}%
}
\vspace{-0.1in}

\caption{Effects of codebook size on Stage 1 reconstruction MSE and $\text{NRMSE}_{\text{sum}}$ under frequency and uniform weighting at a 96-step prediction length.}
\label{tab:results_codereset_main}%
\vskip -0.15in
\end{table}%
Table~\ref{tab:results_codereset_main} shows that $K=64$ yields the highest reconstruction errors (\textcolor{red!15}{red} cells) and forecasting errors under both weighting schemes across all three datasets, suggesting that insufficient representational capacity can limit forecasting quality. Increasing $K$ from $256$ to $512$ provides little additional reconstruction benefit (\textcolor{yellow!100}{yellow} cells), yet forecasting errors increase under frequency weighting, particularly on Solar, and decrease slightly under uniform weighting. One possible explanation is that additional prototypes subdivide existing latent patterns without substantial representational gains, while the resulting probability allocation may make conditional transport more difficult to learn. The contrasting trends under the two weighting schemes suggest an interaction between codebook capacity and source probability allocation. Given the available training data and a fixed ODE sampling budget, these design choices should therefore be considered jointly, as improved reconstruction alone does not guarantee better forecasting.

\vspace{-0.05in}
\subsection{Effects of the prototype prior on transport and sampling-step sensitivity}
Figure~\ref{fig:sampling} examines both the transport geometry and sampling-step sensitivity of ProtoFlow and ProtoFlow-G across six ODE steps on three benchmarks with a prediction horizon of 96. ProtoFlow consistently exhibits a smaller $W_2$ distance than ProtoFlow-G across all sampling budgets, indicating that the prototype-based initialization induces a transport configuration that is better aligned with the target distribution. This observation is consistent with the theoretical motivation in Section~\ref{sec:proto_theory}, where a source prior closer to the target distribution corresponds to a smaller optimal rectified transport cost. ProtoFlow also achieves its best or near-best $\text{CRPS}_{\text{sum}}$ with only a few ODE steps and remains relatively stable as the sampling budget increases.
\begin{figure*}[h]
\vspace{-0.1in}
  \centering
   \includegraphics[width=0.65\linewidth]{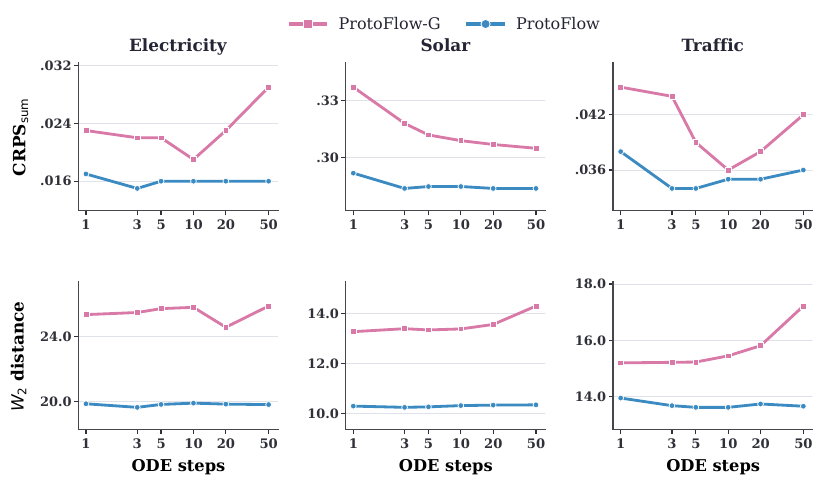}
   \vspace{-0.12in}
   \caption{Effect of ODE sampling steps on 96-step probabilistic forecasting performance and $W_2$ transport distance for ProtoFlow and ProtoFlow-G.}
   \label{fig:sampling}
\end{figure*}

\vspace{-0.18in}
\section{Conclusion}
In this paper, we propose ProtoFlow, a prototype-prior
flow matching framework that integrates vector-quantized
representations with non-autoregressive transport for
high-dimensional multivariate time series forecasting.
By constructing the source distribution from learned
prototypes and their local statistics, ProtoFlow achieves
strong forecasting performance with few ODE steps,
alongside earlier stabilization of flow matching training
and reduced sensitivity to the sampling budget. Beyond these empirical gains, our findings highlight
a broader role for learned codebooks: their structure
can inform both the representation of time series and
the initialization of generative transport.
This perspective motivates closer integration of
representation learning and source-prior design as
a promising direction for efficient generative 
time series forecasting.

\section*{AI use statement}


Generative AI tools assisted with identifying relevant
literature, drafting and refining manuscript text,
LaTeX formatting, providing feedback on experimental
design, and supporting code development.
The authors reviewed all AI-assisted content,
checked scientific claims against primary sources
and experimental evidence, and retain full responsibility
for the accuracy and integrity of the final manuscript.






\section*{Reproducibility Statement}
Section~3 describes the vector-quantized tokenizer,
prototype-prior construction, training objectives,
and inference procedure. The theoretical motivation
is presented in Section~\ref{sec:proto_theory}, with
further derivations in Appendix~\ref{app:proto_theory}.
Section~4 specifies the benchmarks, baselines,
evaluation settings, and ablation studies.
Additional dataset information and metric definitions
are provided in Appendices~\ref{supp:more_dataset_details}
and~\ref{supp:more_impl_details}, respectively.
Stage-specific model configurations and training
hyperparameters are reported in
Tables~\ref{tab:ts_data} and~\ref{tab:ts_data2}.
Experiments use a single NVIDIA A100 GPU with
40\,GB memory, and results are averaged over three
runs, with forecasting metrics computed from
100 generated samples per prediction window. The code will be made publicly available upon acceptance.


\bibliography{iclr2027_conference}

@inproceedings{huang2024generative,
  title={Generative learning for financial time series with irregular and scale-invariant patterns},
  author={Huang, Hongbin and Chen, Minghua and Qiao, Xiao},
  booktitle={The Twelfth International Conference on Learning Representations},
  year={2024}
}

@article{shao2022decoupled,
  title={Decoupled dynamic spatial-temporal graph neural network for traffic forecasting},
  author={Shao, Zezhi and Zhang, Zhao and Wei, Wei and Wang, Fei and Xu, Yongjun and Cao, Xin and Jensen, Christian S},
  journal={arXiv preprint arXiv:2206.09112},
  year={2022}
}

@inproceedings{jiang2023pdformer,
  title={Pdformer: Propagation delay-aware dynamic long-range transformer for traffic flow prediction},
  author={Jiang, Jiawei and Han, Chengkai and Zhao, Wayne Xin and Wang, Jingyuan},
  booktitle={Proceedings of the AAAI conference on artificial intelligence},
  volume={37},
  number={4},
  pages={4365--4373},
  year={2023}
}

@article{wu2020adversarial,
  title={Adversarial sparse transformer for time series forecasting},
  author={Wu, Sifan and Xiao, Xi and Ding, Qianggang and Zhao, Peilin and Wei, Ying and Huang, Junzhou},
  journal={Advances in neural information processing systems},
  volume={33},
  pages={17105--17115},
  year={2020}
}

@article{chen2024sdformer,
  title={Sdformer: Similarity-driven discrete transformer for time series generation},
  author={Chen, Zhicheng and SHIBO, FENG and Zhang, Zhong and Xiao, Xi and Gao, Xingyu and Zhao, Peilin},
  journal={Advances in Neural Information Processing Systems},
  volume={37},
  pages={132179--132207},
  year={2024}
}

@article{wu2023interpretable,
  title={Interpretable weather forecasting for worldwide stations with a unified deep model},
  author={Wu, Haixu and Zhou, Hang and Long, Mingsheng and Wang, Jianmin},
  journal={Nature Machine Intelligence},
  volume={5},
  number={6},
  pages={602--611},
  year={2023},
  publisher={Nature Publishing Group UK London}
}

@article{dinh2016density,
  title={Density estimation using real nvp},
  author={Dinh, Laurent and Sohl-Dickstein, Jascha and Bengio, Samy},
  journal={arXiv preprint arXiv:1605.08803},
  year={2016}
}

@article{song2020score,
  title={Score-based generative modeling through stochastic differential equations},
  author={Song, Yang and Sohl-Dickstein, Jascha and Kingma, Diederik P and Kumar, Abhishek and Ermon, Stefano and Poole, Ben},
  journal={arXiv preprint arXiv:2011.13456},
  year={2020}
}

@article{kingma2019introduction,
  title={An introduction to variational autoencoders},
  author={Kingma, Diederik P and Welling, Max and others},
  journal={Foundations and Trends{\textregistered} in Machine Learning},
  volume={12},
  number={4},
  pages={307--392},
  year={2019},
  publisher={Now Publishers, Inc.}
}

@inproceedings{sohl2015deep,
  title={Deep unsupervised learning using nonequilibrium thermodynamics},
  author={Sohl-Dickstein, Jascha and Weiss, Eric and Maheswaranathan, Niru and Ganguli, Surya},
  booktitle={International conference on machine learning},
  pages={2256--2265},
  year={2015},
  organization={pmlr}
}

@article{wang2024timedart,
  title={Timedart: A diffusion autoregressive transformer for self-supervised time series representation},
  author={Wang, Daoyu and Cheng, Mingyue and Liu, Zhiding and Liu, Qi},
  journal={arXiv preprint arXiv:2410.05711},
  year={2024}
}

@inproceedings{shen2023non,
  title={Non-autoregressive conditional diffusion models for time series prediction},
  author={Shen, Lifeng and Kwok, James},
  booktitle={International Conference on Machine Learning},
  pages={31016--31029},
  year={2023},
  organization={PMLR}
}

@article{cao2024timedit,
  title={Timedit: General-purpose diffusion transformers for time series foundation model},
  author={Cao, Defu and Ye, Wen and Zhang, Yizhou and Liu, Yan},
  journal={arXiv preprint arXiv:2409.02322},
  year={2024}
}

@article{ruan2025vision,
  title={Vision-Enhanced Time Series Forecasting via Latent Diffusion Models},
  author={Ruan, Weilin and Zhong, Siru and Wen, Haomin and Liang, Yuxuan},
  journal={arXiv preprint arXiv:2502.14887},
  year={2025}
}

@article{talukder2024totem,
  title={Totem: Tokenized time series embeddings for general time series analysis},
  author={Talukder, Sabera and Yue, Yisong and Gkioxari, Georgia},
  journal={arXiv preprint arXiv:2402.16412},
  year={2024}
}

@article{feng2025hdt,
  title={Hdt: Hierarchical discrete transformer for multivariate time series forecasting},
  author={Feng, Shibo and Zhao, Peilin and Liu, Liu and Wu, Pengcheng and Shen, Zhiqi},
  journal={arXiv preprint arXiv:2502.08302},
  year={2025}
}

@article{bengio2015scheduled,
  title={Scheduled sampling for sequence prediction with recurrent neural networks},
  author={Bengio, Samy and Vinyals, Oriol and Jaitly, Navdeep and Shazeer, Noam},
  journal={Advances in neural information processing systems},
  volume={28},
  year={2015}
}

@article{schmidt2019generalization,
  title={Generalization in generation: A closer look at exposure bias},
  author={Schmidt, Florian},
  journal={arXiv preprint arXiv:1910.00292},
  year={2019}
}

@article{zhang2024trajectory,
  title={Trajectory flow matching with applications to clinical time series modelling},
  author={Zhang, Xi Nicole and Pu, Yuan and Kawamura, Yuki and Loza, Andrew and Bengio, Yoshua and Shung, Dennis and Tong, Alexander},
  journal={Advances in Neural Information Processing Systems},
  volume={37},
  pages={107198--107224},
  year={2024}
}

@article{kollovieh2024flow,
  title={Flow matching with gaussian process priors for probabilistic time series forecasting},
  author={Kollovieh, Marcel and Lienen, Marten and L{\"u}dke, David and Schwinn, Leo and G{\"u}nnemann, Stephan},
  journal={arXiv preprint arXiv:2410.03024},
  year={2024}
}

@article{lee2023vector,
  title={Vector quantized time series generation with a bidirectional prior model},
  author={Lee, Daesoo and Malacarne, Sara and Aune, Erlend},
  journal={arXiv preprint arXiv:2303.04743},
  year={2023}
}

@article{van2017neural,
  title={Neural discrete representation learning},
  author={Van Den Oord, Aaron and Vinyals, Oriol and others},
  journal={Advances in neural information processing systems},
  volume={30},
  year={2017}
}

@article{liu2022flow,
  title={Flow straight and fast: Learning to generate and transfer data with rectified flow},
  author={Liu, Xingchao and Gong, Chengyue and Liu, Qiang},
  journal={arXiv preprint arXiv:2209.03003},
  year={2022}
}

@article{lipman2022flow,
  title={Flow matching for generative modeling},
  author={Lipman, Yaron and Chen, Ricky TQ and Ben-Hamu, Heli and Nickel, Maximilian and Le, Matt},
  journal={arXiv preprint arXiv:2210.02747},
  year={2022}
}

@inproceedings{peebles2023scalable,
  title={Scalable diffusion models with transformers},
  author={Peebles, William and Xie, Saining},
  booktitle={Proceedings of the IEEE/CVF international conference on computer vision},
  pages={4195--4205},
  year={2023}
}

@inproceedings{dai2019transformer,
  title={Transformer-xl: Attentive language models beyond a fixed-length context},
  author={Dai, Zihang and Yang, Zhilin and Yang, Yiming and Carbonell, Jaime G and Le, Quoc and Salakhutdinov, Ruslan},
  booktitle={Proceedings of the 57th annual meeting of the association for computational linguistics},
  pages={2978--2988},
  year={2019}
}

@inproceedings{lai2018modeling,
  title={Modeling long-and short-term temporal patterns with deep neural networks},
  author={Lai, Guokun and Chang, Wei-Cheng and Yang, Yiming and Liu, Hanxiao},
  booktitle={The 41st international ACM SIGIR conference on research \& development in information retrieval},
  pages={95--104},
  year={2018}
}

@article{salinas2019high,
  title={High-dimensional multivariate forecasting with low-rank gaussian copula processes},
  author={Salinas, David and Bohlke-Schneider, Michael and Callot, Laurent and Medico, Roberto and Gasthaus, Jan},
  journal={Advances in neural information processing systems},
  volume={32},
  year={2019}
}

@inproceedings{gasthaus2019probabilistic,
  title={Probabilistic forecasting with spline quantile function RNNs},
  author={Gasthaus, Jan and Benidis, Konstantinos and Wang, Yuyang and Rangapuram, Syama Sundar and Salinas, David and Flunkert, Valentin and Januschowski, Tim},
  booktitle={The 22nd international conference on artificial intelligence and statistics},
  pages={1901--1910},
  year={2019},
  organization={PMLR}
}

@inproceedings{rasul2021autoregressive,
  title={Autoregressive denoising diffusion models for multivariate probabilistic time series forecasting},
  author={Rasul, Kashif and Seward, Calvin and Schuster, Ingmar and Vollgraf, Roland},
  booktitle={International conference on machine learning},
  pages={8857--8868},
  year={2021},
  organization={PMLR}
}

@article{tashiro2021csdi,
  title={Csdi: Conditional score-based diffusion models for probabilistic time series imputation},
  author={Tashiro, Yusuke and Song, Jiaming and Song, Yang and Ermon, Stefano},
  journal={Advances in neural information processing systems},
  volume={34},
  pages={24804--24816},
  year={2021}
}

@article{kollovieh2023predict,
  title={Predict, refine, synthesize: Self-guiding diffusion models for probabilistic time series forecasting},
  author={Kollovieh, Marcel and Ansari, Abdul Fatir and Bohlke-Schneider, Michael and Zschiegner, Jasper and Wang, Hao and Wang, Yuyang Bernie},
  journal={Advances in Neural Information Processing Systems},
  volume={36},
  pages={28341--28364},
  year={2023}
}

@article{fan2024mg,
  title={Mg-tsd: Multi-granularity time series diffusion models with guided learning process},
  author={Fan, Xinyao and Wu, Yueying and Xu, Chang and Huang, Yuhao and Liu, Weiqing and Bian, Jiang},
  journal={arXiv preprint arXiv:2403.05751},
  year={2024}
}

@article{guo2025dynamical,
  title={Dynamical diffusion: Learning temporal dynamics with diffusion models},
  author={Guo, Xingzhuo and Zhang, Yu and Chen, Baixu and Xu, Haoran and Wang, Jianmin and Long, Mingsheng},
  journal={arXiv preprint arXiv:2503.00951},
  year={2025}
}

@article{ye2025non,
  title={Non-stationary Diffusion For Probabilistic Time Series Forecasting},
  author={Ye, Weiwei and Xu, Zhuopeng and Gui, Ning},
  journal={arXiv preprint arXiv:2505.04278},
  year={2025}
}

@article{el2025probabilistic,
  title={Probabilistic Forecasting via Autoregressive Flow Matching},
  author={El-Gazzar, Ahmed and van Gerven, Marcel},
  journal={arXiv preprint arXiv:2503.10375},
  year={2025}
}

@article{matheson1976scoring,
  title={Scoring rules for continuous probability distributions},
  author={Matheson, James E and Winkler, Robert L},
  journal={Management science},
  volume={22},
  number={10},
  pages={1087--1096},
  year={1976},
  publisher={INFORMS}
}

@inproceedings{kim2021reversible,
  title={Reversible instance normalization for accurate time-series forecasting against distribution shift},
  author={Kim, Taesung and Kim, Jinhee and Tae, Yunwon and Park, Cheonbok and Choi, Jang-Ho and Choo, Jaegul},
  booktitle={International conference on learning representations},
  year={2021}
}

@article{meijer2024rise,
  title={The rise of diffusion models in time-series forecasting},
  author={Meijer, Caspar and Chen, Lydia Y},
  journal={arXiv preprint arXiv:2401.03006},
  year={2024}
}

@article{tong2023improving,
  title={Improving and generalizing flow-based generative models with minibatch optimal transport},
  author={Tong, Alexander and Fatras, Kilian and Malkin, Nikolay and Huguet, Guillaume and Zhang, Yanlei and Rector-Brooks, Jarrid and Wolf, Guy and Bengio, Yoshua},
  journal={arXiv preprint arXiv:2302.00482},
  year={2023}
}

@article{alexandrov2020gluonts,
  title={Gluonts: Probabilistic and neural time series modeling in python},
  author={Alexandrov, Alexander and Benidis, Konstantinos and Bohlke-Schneider, Michael and Flunkert, Valentin and Gasthaus, Jan and Januschowski, Tim and Maddix, Danielle C and Rangapuram, Syama and Salinas, David and Schulz, Jasper and others},
  journal={Journal of Machine Learning Research},
  volume={21},
  number={116},
  pages={1--6},
  year={2020}
}

@article{paszke2019pytorch,
  title={Pytorch: An imperative style, high-performance deep learning library},
  author={Paszke, Adam and Gross, Sam and Massa, Francisco and Lerer, Adam and Bradbury, James and Chanan, Gregory and Killeen, Trevor and Lin, Zeming and Gimelshein, Natalia and Antiga, Luca and others},
  journal={Advances in neural information processing systems},
  volume={32},
  year={2019}
}

@inproceedings{rasul2024vq,
  title={Vq-tr: Vector quantized attention for time series forecasting},
  author={Rasul, Kashif and Bennett, Andrew and Vicente, Pablo and Gupta, Umang and Ghonia, Hena and Schneider, Anderson and Nevmyvaka, Yuriy},
  booktitle={International Conference on Learning Representations},
  volume={2024},
  pages={47137--47158},
  year={2024}
}

@article{feng2026msdformer,
  title={MSDformer: Multi-scale Discrete Transformer For Time Series Generation},
  author={Feng, Shibo and Chen, Zhicheng and Xiao, Xi and Zhang, Zhong and Li, Qing and Gao, Xingyu and Zhao, Peilin},
  journal={IEEE Transactions on Pattern Analysis and Machine Intelligence},
  year={2026},
  publisher={IEEE}
}

@article{li2026sdflow,
  title={SDFlow: Similarity-Driven Flow Matching for Time Series Generation},
  author={Li, Wei and Feng, Shibo and Wu, Pengcheng and Gao, Xingyu and Wu, Min and Zhao, Peilin},
  journal={arXiv preprint arXiv:2605.05736},
  year={2026}
}

@inproceedings{kollovieh2025flow,
  title={Flow matching with gaussian process priors for probabilistic time series forecasting},
  author={Kollovieh, Marcel and Lienen, Marten and L{\"u}dke, David and Schwinn, Leo and G{\"u}nnemann, Stephan},
  booktitle={International Conference on Learning Representations},
  volume={2025},
  pages={95435--95457},
  year={2025}
}
\bibliographystyle{iclr2027_conference}
\appendix
\newpage
\begin{center}
    {\LARGE \bf Supplementary Materials for ProtoFlow}\\[4pt]
\end{center}
\section{Additional Theoretical Analysis of Codebook-Prior Initialization}
\label{app:proto_theory}

In this section, we provide the formal derivation underlying the theoretical motivation for codebook-prior initialization in ProtoFlow.

\subsection{Optimal Rectified Transport Cost and Wasserstein Distance}

Let $\mathcal P_1$ denote the target latent distribution of $\mathbf{x}_s^1$, and let $\mathcal P_0$ be an arbitrary source prior defined on the same latent space $\mathbb{R}^{L\times n_z}$. Assume that both $\mathcal P_0$ and $\mathcal P_1$ have finite second moments. For any coupling
\[
\Gamma \in \Pi(\mathcal P_0,\mathcal P_1),
\]
where $\Pi(\mathcal P_0,\mathcal P_1)$ denotes the set of all joint distributions with marginals $\mathcal P_0$ and $\mathcal P_1$, define the rectified interpolation path
\begin{equation}
\mathbf{x}_s^t = (1-t)\mathbf{x}_s^0 + t \mathbf{x}_s^1,
\qquad t\in[0,1].
\label{eq:app_rectified_interp}
\end{equation}
The expected squared rectified transport cost under coupling $\Gamma$ is defined as
\begin{equation}
\mathcal C_{\mathrm{rect}}(\Gamma)
:=
\mathbb E_{(\mathbf{x}_s^0,\mathbf{x}_s^1)\sim\Gamma}
\left[
\|\mathbf{x}_s^1-\mathbf{x}_s^0\|_F^2
\right].
\label{eq:app_rect_cost}
\end{equation}

We further define the minimum expected squared rectified transport cost from $\mathcal P_0$ to $\mathcal P_1$ as
\begin{equation}
\mathcal C_{\mathrm{opt}}(\mathcal P_0,\mathcal P_1)
:=
\inf_{\Gamma\in\Pi(\mathcal P_0,\mathcal P_1)}
\mathcal C_{\mathrm{rect}}(\Gamma).
\label{eq:app_opt_cost}
\end{equation}

\begin{proposition}
\label{prop:app_w2}
Under the above definitions,
\begin{equation}
\mathcal C_{\mathrm{opt}}(\mathcal P_0,\mathcal P_1)
=
W_2^2(\mathcal P_0,\mathcal P_1),
\label{eq:app_w2_result}
\end{equation}
where $W_2(\mathcal P_0,\mathcal P_1)$ is the $2$-Wasserstein distance under the Frobenius ground cost.
\end{proposition}

\begin{proof}
For a fixed pair $(\mathbf{x}_s^0,\mathbf{x}_s^1)$, the rectified path is given by Eq.~\eqref{eq:app_rectified_interp}. Differentiating with respect to $t$ gives
\begin{equation}
\frac{d\mathbf{x}_s^t}{dt}=\mathbf{x}_s^1-\mathbf{x}_s^0.
\label{eq:app_velocity}
\end{equation}
Since the velocity is constant along the rectified path, its path length is
\begin{equation}
\int_0^1
\left\|
\frac{d\mathbf{x}_s^t}{dt}
\right\|_F dt
=
\int_0^1
\|\mathbf{x}_s^1-\mathbf{x}_s^0\|_F dt
=
\|\mathbf{x}_s^1-\mathbf{x}_s^0\|_F.
\label{eq:app_length_single}
\end{equation}
Hence, the squared rectified path length is
\begin{equation}
\|\mathbf{x}_s^1-\mathbf{x}_s^0\|_F^2.
\end{equation}
Taking expectation under a coupling $\Gamma\in\Pi(\mathcal P_0,\mathcal P_1)$ gives
\begin{equation}
\mathcal C_{\mathrm{rect}}(\Gamma)
=
\mathbb E_{(\mathbf{x}_s^0,\mathbf{x}_s^1)\sim\Gamma}
\left[
\|\mathbf{x}_s^1-\mathbf{x}_s^0\|_F^2
\right].
\end{equation}
By definition, the squared $2$-Wasserstein distance between $\mathcal P_0$ and $\mathcal P_1$ under the Frobenius ground cost is
\begin{equation}
W_2^2(\mathcal P_0,\mathcal P_1)
=
\inf_{\Gamma\in\Pi(\mathcal P_0,\mathcal P_1)}
\mathbb E_{(\mathbf{x}_s^0,\mathbf{x}_s^1)\sim\Gamma}
\left[
\|\mathbf{x}_s^1-\mathbf{x}_s^0\|_F^2
\right].
\end{equation}
Combining this with Eq.~\eqref{eq:app_opt_cost} yields
\[
\mathcal C_{\mathrm{opt}}(\mathcal P_0,\mathcal P_1)
=
W_2^2(\mathcal P_0,\mathcal P_1).
\]
\end{proof}

\subsection{Implication for Codebook-Prior Initialization}

Let $\mathcal P_C$ denote the source distribution induced by the learned codebook prior used in ProtoFlow, and let
\begin{equation}
\mathcal P_G := \mathcal N(0,I_{Ln_z})
\end{equation}
denote the Gaussian prior in the same latent space.

\begin{corollary}
\label{cor:app_prior}
If the learned codebook prior is closer to the target latent distribution than the Gaussian prior in $2$-Wasserstein distance, namely
\begin{equation}
W_2(\mathcal P_C,\mathcal P_1)
<
W_2(\mathcal P_G,\mathcal P_1),
\label{eq:app_assumption}
\end{equation}
then
\begin{equation}
\mathcal C_{\mathrm{opt}}(\mathcal P_C,\mathcal P_1)
<
\mathcal C_{\mathrm{opt}}(\mathcal P_G,\mathcal P_1).
\label{eq:app_cor_result}
\end{equation}
Thus, under optimal source--target coupling, the codebook
prior yields a smaller expected squared path length for
the linear interpolation.
\end{corollary}

\begin{proof}
By Proposition~\ref{prop:app_w2},
\begin{equation}
\mathcal C_{\mathrm{opt}}(\mathcal P_C,\mathcal P_1)
=
W_2^2(\mathcal P_C,\mathcal P_1),
\end{equation}
and
\begin{equation}
\mathcal C_{\mathrm{opt}}(\mathcal P_G,\mathcal P_1)
=
W_2^2(\mathcal P_G,\mathcal P_1).
\end{equation}
Under the assumption in Eq.~\eqref{eq:app_assumption}, and since both Wasserstein distances are nonnegative, squaring preserves the strict inequality:
\begin{equation}
W_2^2(\mathcal P_C,\mathcal P_1)
<
W_2^2(\mathcal P_G,\mathcal P_1).
\end{equation}
Substituting the above equalities gives Eq.~\eqref{eq:app_cor_result}.
\end{proof}

Corollary~\ref{cor:app_prior} shows that the advantage of codebook-prior initialization can be interpreted geometrically: whenever the learned codebook prior is closer to the target latent distribution than Gaussian initialization in Wasserstein distance, it yields a smaller optimal rectified transport cost.

\section{Related Work}
\label{sec:related_work}

We discuss two lines of research underlying ProtoFlow: vector-quantized latent modeling and diffusion/flow
matching for multivariate time series forecasting.

\noindent\textbf{Vector-quantized Latent Modeling.}
Vector quantization~\cite{van2017neural} maps continuous representations to a
finite set of learned embeddings, enabling discrete latent
modeling. Its promising performance in time series forecasting and generation has motivated growing interest in this representation strategy. For example, in time series forecasting task, VQ-TR
\cite{rasul2024vq} introduces quantization within attention for computationally efficient forecasting and HDT~\cite{feng2025hdt} formulates multivariate forecasting
through hierarchical discrete token modeling. Related generation methods explore diverse latent structures and generation strategies within VQ-based time series modeling. TimeVQVAE~\cite{lee2023vector} employs bidirectional prior modeling, while SDformer~\cite{chen2024sdformer} combines similarity-driven vector quantization with autoregressive or masked token modeling. MSDformer~\cite{feng2026msdformer} extends discrete modeling through multi-scale tokenization and autoregressive generation. SDFlow~\cite{li2026sdflow} instead performs non-autoregressive generation through variational flow matching in a frozen VQ latent space, incorporating categorical supervision over codebook indices. In addition, TOTEM~\cite{talukder2024totem} learns a shared vector-quantized tokenizer across diverse domains, enabling transferable discrete representations for forecasting, imputation, and anomaly detection tasks.  ProtoFlow additionally uses the codebook
as a source of distributional information, constructing
a prototype prior for continuous transport in the latent
embedding space.

\noindent\textbf{Diffusion and Flow Matching for Forecasting.}
Diffusion-based methods have demonstrated strong performance in time series analysis, particularly in multivariate forecasting, where iterative denoising enables flexible modeling of future uncertainty.
TimeGrad~\cite{rasul2021autoregressive}
combines diffusion with autoregressive prediction, while
CSDI~\cite{tashiro2021csdi} and
TimeDiff~\cite{shen2023non} investigate conditional
reconstruction and non-autoregressive prediction,
respectively. TSDiff~\cite{kollovieh2023predict} enables
forecasting with an unconditional diffusion model through
inference-time guidance. Subsequent methods incorporate
multi-granularity supervision MG-TSD
\cite{fan2024mg}, temporal dynamics DyDiff
\cite{guo2025dynamical}, and time-varying uncertainty NsDiff
\cite{ye2025non}. Although effective, iterative sampling
can impose substantial inference costs.

Flow matching learns transport from a source distribution to the data distribution through vector-field regression, enabling ODE-based generation that can support accurate sampling with fewer integration steps through suitable probability-path design~\cite{lipman2022flow,tong2023improving}.
For time series forecasting task, FlowTime
\cite{el2025probabilistic} develops autoregressive flow
matching in observation space, whereas TSFlow
\cite{kollovieh2025flow} incorporates Gaussian-process
priors to encode temporal structure and supports both
conditional and unconditional modeling. These approaches
motivate studying both transport dynamics and source
distributions. ProtoFlow explores this direction by constructing a source prior from learned latent prototypes, supporting accurate forecasting with few ODE sampling steps while facilitating earlier stabilization of flow matching training.

\noindent\textbf{Positioning of ProtoFlow.}
ProtoFlow connects vector-quantized representations with
prototype-prior flow matching. Its codebook serves two
roles: representing target sequences and defining a
structured source distribution through prototype centers,
empirical frequencies, and local variability. Conditioned
on historical context, a non-autoregressive vector field
jointly transports the future latent sequence, which is
then decoded into forecasts. \textbf{The distinguishing design is the reuse of learned representation structure to initialize latent transport, providing a data-informed source prior that supports earlier FM training stabilization and accurate forecasting with few sampling steps.} Our experiments examine
how this initialization affects FM training dynamics, few-step forecasting, and sensitivity to the sampling budget.

\section{More Implementation Details}
\label{imp_appendix}

\subsection{More Details on Datasets}\label{supp:more_dataset_details}
We conduct extensive evaluations of the proposed ProtoFlow framework on five widely-used real-world benchmark datasets for multivariate time series forecasting, covering representative application domains \cite{lai2018modeling} including \emph{Energy} (Solar and Electricity), \emph{Traffic} (Traffic and Taxi), and \emph{Society and Economics} (Wikipedia). Table~\ref{tab:VQ_data} summarizes the key characteristics of these datasets. All experiments are repeated three times, implemented in PyTorch \cite{paszke2019pytorch} and GluonTS \cite{alexandrov2020gluonts}.

\begin{table}[h]
  \centering
\resizebox{1.0\linewidth}{!}{\begin{tabular}{ccccccc}
\hline
  \textbf{DATASET} & Dimension & Domain & Freq & Total Time Steps & Pred Length & Rolling Windows \\ \hline
SOLAR  & 137  & $\mathbb{R}^{+}$ & Hourly & 7,009 & $\{48, 96\}$ & 7\\ 
ELECTRICITY  & 370 & $\mathbb{R}^{+}$ & Hourly & 5,790 & $\{48, 96\}$  & 7\\  
TRAFFIC   & 963 & (0,1) & Hourly & 10,413 & $\{48, 96\}$ & 7 \\  
TAXI    & 1214  & $\mathbb{N}$ & 30-Min & 1,488 & $\{48, 96\}$ & 57 \\ 
WIKIPEDIA   & 2000  & $\mathbb{N}$ & Daily & 792 & $\{48, 96\}$ & 5\\ \hline
\end{tabular}
}
  \caption{The dataset properties used in our experiments are summarized as follows.}
  \label{tab:VQ_data}
\end{table}
Specifically, the sources of these datasets are shown in the following:
\begin{itemize}
	\item \textbf{\small SOLAR } records 10-minute solar power production from 137 PV plants in Alabama in 2006;
    \vspace{-0.05in}
    \item \textbf{\small ELECTRICITY} comprises electricity consumption for $n=370$ clients measured at 15-minute intervals during 2012--2014;
    \vspace{-0.05in}
	\item \textbf{\small TRAFFIC} provides hourly road occupancy rates (normalized to $[0,1]$) collected by highway sensors in the San Francisco Bay Area over multiple years;
    \vspace{-0.05in}
	\item \textbf{\small TAXI} is a spatio-temporal series of New York City taxi rides aggregated every 30 minutes at 1214 locations, using January 2015 for training and January 2016 for testing as shown in GluonTS \cite{alexandrov2020gluonts};    \vspace{-0.05in}
	\item \textbf{\small WIKIPEDIA} contains daily page-view counts for 2000 Wikipedia pages of GluonTS repository \cite{alexandrov2020gluonts};
\end{itemize}

Following standard practice, we adopt a strictly chronological split based on a fixed cutoff time, partitioning each dataset into 60\% training, 20\% validation, and 20\% test. We further evaluate with a rolling-window protocol on the test segment: forecasts are initiated from multiple start times that are evenly spaced after the last training timestamp, and metrics are aggregated across windows. For hourly datasets, we use 7 rolling windows; for the remaining datasets we use 5 windows, except for Taxi where 57 windows are employed to cover the full test period.

\subsection{More Evaluation Protocols}\label{supp:more_impl_details}
We evaluate multivariate forecasting using
$\text{CRPS}_{\text{sum}}$ and
$\text{NRMSE}_{\text{sum}}$, which measure probabilistic
and point forecasting accuracy on the series obtained by
summing across variables at each time step. These metrics
assess how well predictions capture the aggregate behavior
of the \textbf{multivariate time series setting}. Below, we describe their
computation using $N=100$ generated forecast samples and
report dataset-specific hyperparameters for Stage 1 and
Stage 2.

\noindent \textbf{Metrics.} CRPS measures the compatibility of a cumulative distribution function $P$ with an observation $x$ as: $\operatorname{CRPS}(\mathcal{F}, x)=\int_{\mathbb{R}}(P(y)-\mathbb{I}\{x \leq y\})^{2} d y,$ where $\mathbb{I}\{x \leq y\}$is the indicator function which is one if $x \leq y$ and zero otherwise. The empirical CDF of $P$, i.e., $\hat{P}(y)=\frac{1}{N} \sum_{i=1}^{N} \mathbb{I}\left\{X_{i} \leq y\right\}$ with n samples $X_{i} \sim P$ as the approximation of the predictive CDF. It utilizes $N$ samples to estimate the empirical CDF and take the $\text{CRPS}_{\text{sum}}$ in the multivariate case.
\begin{equation}
\vspace{-0.1in}
\operatorname{CRPS_{\text {sum }}}=\mathbb{E}_{t}\left[\operatorname{CRPS}\left(\widehat{P}_{\text {sum }}(t), \sum_{i} x_{i}^{t}\right)\right].
\end{equation} \\
The Normalized Root Mean Squared Error (NRMSE) is a standardized version of the Root Mean Squared Error (RMSE) that accounts for the scale of the target values. The formula for NRMSE is given below:
\begin{equation}
\text{NRMSE}=\sqrt{\frac{1}{T} \sum_{t=1}^T\left(\frac{y_{t}-\hat{y_{t}}}{y_{max}-y_{min}}\right)^2}, 
\end{equation}
where $\hat{y_{t}}$ represents the predicted target, and $y_{t}$ represents the true target. $y_{max}$ and $y_{min}$ are the minimum and maximum of the measured target values, respectively. The NRMSE quantifies the average squared discrepancy between the predictions and actual observations, normalized by the range of the target values and $\text{NRMSE}_{\text{sum}}$ is the sum of the NRMSE in our 100 samples experimental setting. A lower NRMSE indicates higher predictive accuracy.

In this section, we report the detailed hyperparameter settings of ProtoFlow for both Stage 1 and Stage 2. In Table~\ref{tab:ts_data}, we list the codebook size and codebook dimension of the vector-quantization strategy for each dataset, together with the EMA decay coefficient ($\lambda$) used to smooth codebook updates and the coefficient ($\beta$) in Eqn.~\ref{eq:vq_loss_main}, which controls the trade-off between the commitment loss and the reconstruction loss. 
\begin{table}[h]
\centering
\resizebox{0.70\linewidth}{!}{
\begin{tabular}{lccccc}
\toprule
Datasets & Codebook Size & Codebook dim & Enc/Dec Layers & $\lambda$ & $\beta$ \\
\midrule
Solar &  256 & 128 & 2 & 0.99 & 0.1 \\
Electricity  &  256 & 256 & 2 &  0.99 &  0.1  \\
Traffic   &  256 & 256 & 2 & 0.99 &  0.1  \\
Taxi   &  256 & 512 & 2 &  0.99  & 0.05 \\
Wikipedia & 256 & 512 & 2 & 0.99 & 0.05 \\
\bottomrule
\end{tabular}
}
\caption{Detailed hyperparameters of Stage 1.}
\label{tab:ts_data}
\end{table}

For Stage 2, we report the number of layers and hidden dimension of $\text{DiT}_{\theta}$, the training batch size and learning rate, and the number of ODE sampling steps used at inference, which are shown in Table~\ref{tab:ts_data2}.
\begin{table}[h]
\centering
\resizebox{0.70\linewidth}{!}{
\begin{tabular}{lccccc}
\toprule
Datasets & Hidden dim & $\text{DiT}_{\theta}$ layers & Batch size & Steps & lr \\
\midrule
Solar &  128 & 4 & 64 & 3 & 5e-4 \\
Electricity  &  256 & 4 & 64 &  3 &  5e-4  \\
Traffic   &  256 & 5 & 64 & 3 &  3e-4  \\
Taxi   &  512 & 5 & 64 &  3  & 3e-4 \\
Wikipedia & 512 & 4 & 64 & 5 & 3e-4 \\
\bottomrule
\end{tabular}
}
\caption{Detailed hyperparameters of Stage 2.}
\label{tab:ts_data2}
\end{table}

\subsection{The Details of ProtoFlow}
\label{supp:details-of-VQFlow}

This section provides comprehensive details of ProtoFlow, including the architectural design of Stage 1 and Stage 2 and the full training and inference procedures.

\noindent\textbf{Overview.} We first describe the model architecture of ProtoFlow, including the design of the vector-quantized latent space (Stage 1) and the flow matching generator (Stage 2). We then provide pseudocode for the training and inference procedures.

\noindent\textbf{Stage 1: Vector-Quantized Latent Representation.} We detail the architectures of the Stage 1 model in Tables~\ref{tab:stage1_encoder} and ~\ref{tab:stage1_decoder}. Stage 1 consists of an encoder, a vector quantization module, and a decoder, which together learn a compact discrete latent representation of the input time series. The tables summarize the network structure, codebook design, and training configurations.

\begin{table}[h]
 \begin{center}
    \resizebox{0.95\linewidth}{!}{
 	\begin{tabular}{c|c|c}
 	\multicolumn{1}{c}{Layer} & \multicolumn{1}{c}{Function} & \multicolumn{1}{c}{Descriptions} \\
        \toprule
 	 1 & Convolution & input channel=$d$, output channel=D, kernel size=3, stride=1, padding=1\\
          2   & ReLU & nn.ReLU() \\
       3 & Convolution & input channel=D, output channel=D, kernel size=4, stride=2, padding=1\\
       4 &ResNet &input channel=D, depth=3, dilation growth rate=3 \\
       5   & ReLU & nn.ReLU() \\
       6 & Convolution & input channel=D, output channel=D, kernel size=4, stride=2, padding=1\\
       7 &ResNet &input channel=D, depth=3, dilation growth rate=3 \\
       8   & ReLU & nn.ReLU() \\
       9 & Convolution & input channel= $d_c$, output channel=H, kernel size=3, stride=1, padding=1\\
    \bottomrule
	\end{tabular}
	}
	 \end{center}
  \caption{
    The detailed architecture of the time serie tokenizer's encoder.
    }
  \label{tab:stage1_encoder}
\end{table}
\begin{table}[h]
 \begin{center}
    \resizebox{0.95\linewidth}{!}{
 	\begin{tabular}{c|c|c}
 	\multicolumn{1}{c}{Layer} & \multicolumn{1}{c}{Function} & \multicolumn{1}{c}{Descriptions} \\
        \toprule
 	 1 & Convolution & input channel=$d_c$, output channel=D, kernel size=3, stride=1, padding=1\\
          2   & ReLU & nn.ReLU() \\
        3 &ResNet &input channel=D, depth=3, dilation growth rate=3 \\
                4   & ReLU & nn.ReLU() \\
        5 &Upsample &nn.Upsample()\\
       6 & Convolution & input channel=D, output channel=D, kernel size=3, stride=1, padding=1\\
       7 &ResNet &input channel=D, depth=3, dilation growth rate=3 \\
       8   & ReLU & nn.ReLU() \\
       9 &Upsample &nn.Upsample()\\
       10 & Convolution & input channel=D, output channel=D, kernel size=3, stride=1, padding=1\\
        11   & ReLU & nn.ReLU() \\
        12 & Convolution & input channel=D, output channel=D, kernel size=3, stride=1, padding=1\\
        13   & ReLU & nn.ReLU() \\
        14 & Convolution & input channel=D, output channel=$d$, kernel size=3, stride=1, padding=1\\
    \bottomrule
	\end{tabular}
	}
	 \end{center}
  \caption{
    The detailed architecture of the time serie tokenizer's decoder.
    } 
  \label{tab:stage1_decoder}
\end{table}
\noindent\textbf{Stage 2: Flow Matching with $\text{DiT}_{\theta}$.} The architecture of the flow matching generator is illustrated in Figure~\ref{fig:details}. The figure provides a detailed view of the transformer blocks, conditioning mechanisms, and adaptive normalization used in Stage 2.

\noindent\textbf{Training and Inference Procedures.} We present the pseudocode of training and inference procedures for ProtoFlow in Algorithm~\ref{alg:protoflow_stage1}, \ref{alg:protoflow_stage2}, and \ref{alg:protoflow_inference}.

\begin{figure*}[!t]
\vspace{-0.1in}
  \centering
   \includegraphics[width=1.0\linewidth]{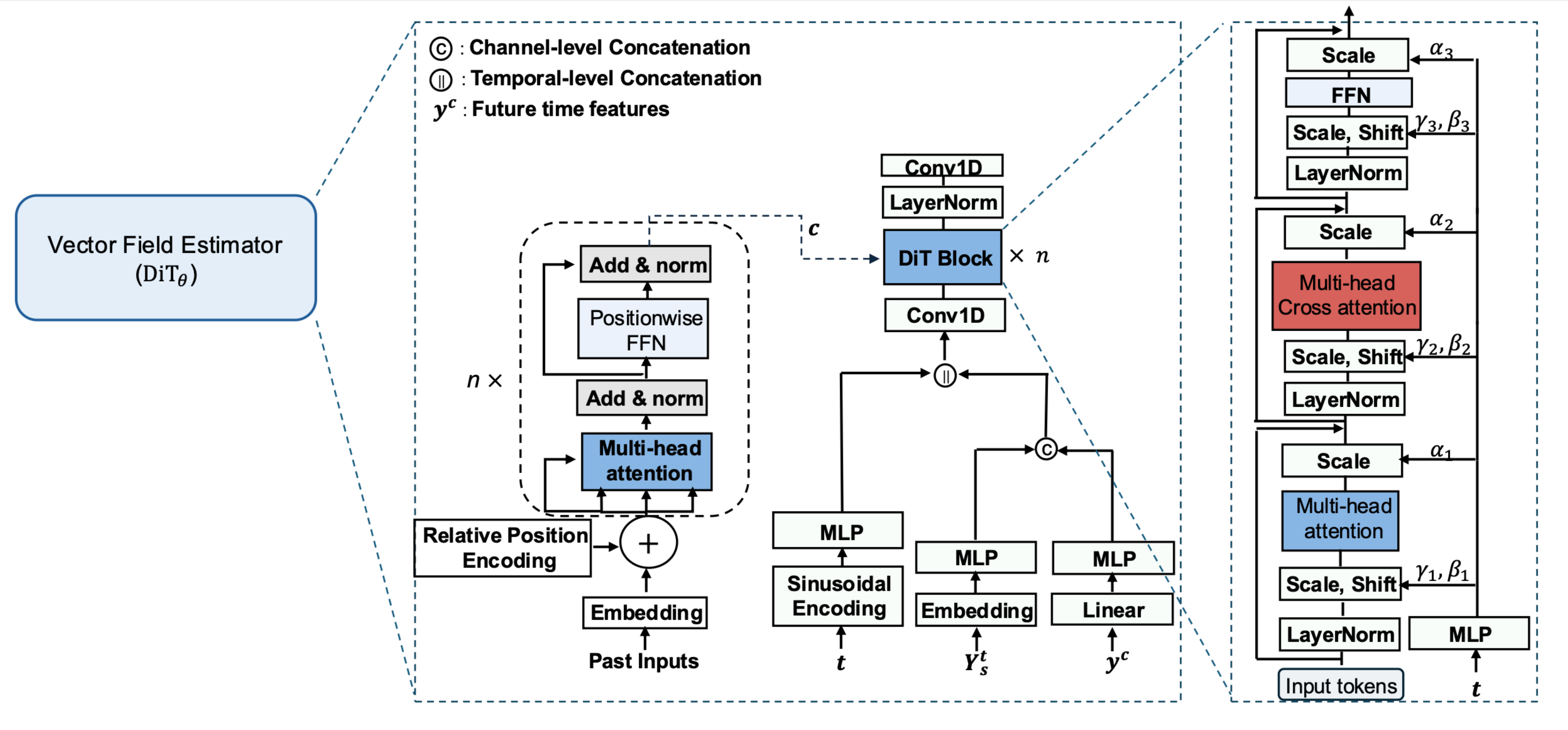}
   \vspace{-0.1in}
   \caption{Illustration of model architecture of $\text{DiT}_{\theta}$ of Stage2.}
   \label{fig:details}
\end{figure*}

\noindent\textbf{Codebook Reset Strategy in Stage 1.}
We maintain EMA statistics of assignment counts and
assigned feature sums for each codebook entry.
Let $u_k$ denote the smoothed assignment count and
$\mathcal{Z}_k^{\mathrm{EMA}}$ the normalized EMA
feature sum. To mitigate codebook collapse, we
reinitialize underutilized entries as
\begin{equation}
\mathcal{Z}_k
\leftarrow
\begin{cases}
h_k^{\mathrm{reset}},
& u_k<u_{\mathrm{th}},\\
\mathcal{Z}_k^{\mathrm{EMA}},
& u_k\geq u_{\mathrm{th}},
\end{cases}
\label{eq:reset}
\end{equation}
where $u_{\mathrm{th}}=1$ and
$h_k^{\mathrm{reset}}$ is a normalized replacement
candidate from the current batch.
Candidates are constructed from the first $K$
flattened encoder outputs. If fewer than $K$ are
available, we repeat them and add Gaussian noise
with per-coordinate standard deviation
$0.01/\sqrt{n_z}$ before selecting and normalizing
the first $K$ vectors.
Resetting replaces the codebook entries while
preserving their EMA statistics and is performed
only during Stage 1 training.

\begin{algorithm}[h]
\caption{ProtoFlow Stage 1: Vector-Quantized Tokenizer Training}
\label{alg:protoflow_stage1}
\begin{algorithmic}[1]
    \STATE \textbf{Input:} Training target sequences
    $\{\mathcal{X}_t\}$
    \STATE \textbf{Parameters:} Codebook size $K$,
    embedding dimension $n_z$, downsampling ratio $\rho$,
    commitment weight $\beta$
    \STATE \textbf{Output:} Encoder $\mathcal{E}_{\psi}$,
    decoder $\mathcal{G}_{\phi}$, codebook $\mathcal{Z}$

    \STATE Initialize $\mathcal{E}_{\psi}$,
    $\mathcal{G}_{\phi}$, and unit-normalized
    codebook $\mathcal{Z}$
    \WHILE{training is not complete}
        \STATE Sample a minibatch
        $\{\mathcal{X}_{t,b}\}_{b\in\mathcal{B}}$
        \STATE Encode and normalize latent vectors:
        $\hat{z}_b\gets
        \mathrm{Norm}(\mathcal{E}_{\psi}(\mathcal{X}_{t,b}))$
        \STATE For all $b,\ell$, assign
        $a_{b,\ell}\gets
        \arg\max_k
        \langle\hat{z}_b^\ell,\mathcal{Z}_k\rangle$
        \STATE Retrieve quantized embeddings:
        $z_{q,b}^{\ell}\gets\mathcal{Z}_{a_{b,\ell}}$
        \STATE Reconstruct:
        $\hat{\mathcal{X}}_{t,b}\gets
        \mathcal{G}_{\phi}(z_{q,b})$
        \STATE Compute the minibatch reconstruction
        and commitment loss using Eq.~\eqref{eq:vq_loss_main}
        \STATE Update $\psi,\phi$ using
        straight-through gradients through quantization
        \STATE Update $\mathcal{Z}$ using EMA,
        reset inactive entries, and normalize codebook entries
    \ENDWHILE
    \STATE Freeze $\mathcal{E}_{\psi}$,
    $\mathcal{G}_{\phi}$, and $\mathcal{Z}$
    \RETURN $\mathcal{E}_{\psi},
    \mathcal{G}_{\phi},\mathcal{Z}$
\end{algorithmic}
\end{algorithm}

\begin{algorithm}[htp]
\caption{ProtoFlow Stage 2: Codebook-Prior Flow Matching}
\label{alg:protoflow_stage2}
\begin{algorithmic}[1]
    \STATE \textbf{Input:} Training data
    $\mathcal{D}=\{(\mathcal{X}_p,\mathcal{X}_t,\mathbf{y}^{c})\}$,
    frozen encoder $\mathcal{E}_{\psi}$ and codebook $\mathcal{Z}$
    \STATE \textbf{Parameters:} Local rank $r$,
    perturbation scale $\gamma$,
    flow-time distribution $q_t$
    \STATE \textbf{Output:} Vector field $\mathbf{v}_{\theta}$,
    context encoder $\mathcal{H}_{\omega}$,
    prior parameters $\mathcal{P}$

    \STATE \textcolor{blue}{\textbf{// Construct the source prior}}
    \STATE Collect normalized encoder outputs and
    code assignments from training targets
    \STATE Estimate
    $\mathcal{P}=\{(\pi_k,\mu_k,B_k,\sigma_k)\}_{k=1}^{K}$
    from $\mathcal{Z}$ and the collected statistics
    \STATE Fix $\mathcal{P}$ and initialize
    $\mathbf{v}_{\theta},\mathcal{H}_{\omega}$

    \STATE \textcolor{blue}{\textbf{// Train conditional latent transport}}
    \WHILE{training is not complete}
        \STATE Sample a minibatch
        $\{(\mathcal{X}_{p,b},\mathcal{X}_{t,b},
        \mathbf{y}^{c}_b)\}_{b\in\mathcal{B}}$
        \STATE For each $b$, obtain the frozen target:
        $\mathbf{x}_{s,b}^{1}\gets
        q(\mathcal{E}_{\psi}(\mathcal{X}_{t,b}))$
        \STATE Construct conditions:
        $\mathbf{c}_b\gets
        (\mathcal{H}_{\omega}(\mathcal{X}_{p,b}),
        \mathbf{y}^{c}_b)$
        \STATE Independently sample
        $\mathbf{x}_{s,b}^{0}\sim P_{\mathcal{Z}}$
        using Eqs.~\eqref{eq:proto_source_position}
        and~\eqref{eq:proto_source}, and sample $t_b\sim q_t$
        \STATE Interpolate:
        $\mathbf{x}_{s,b}^{t_b}\gets
        (1-t_b)\mathbf{x}_{s,b}^{0}
        +t_b\mathbf{x}_{s,b}^{1}$
        \STATE Compute the velocity regression loss:
        \[
        \mathcal{L}_{\mathrm{FM}}
        \gets
        \frac{1}{|\mathcal{B}|}
        \sum_{b\in\mathcal{B}}
        \left\|
        \mathbf{v}_{\theta}
        (\mathbf{x}_{s,b}^{t_b},t_b\mid\mathbf{c}_b)
        -
        (\mathbf{x}_{s,b}^{1}-\mathbf{x}_{s,b}^{0})
        \right\|_F^2
        \]
        \STATE Update $\theta,\omega$ by minimizing
        $\mathcal{L}_{\mathrm{FM}}$
    \ENDWHILE
    \RETURN $\mathbf{v}_{\theta},
    \mathcal{H}_{\omega},\mathcal{P}$
\end{algorithmic}
\end{algorithm}

\begin{algorithm}[htp]
\caption{ProtoFlow Inference}
\label{alg:protoflow_inference}
\begin{algorithmic}[1]
    \STATE \textbf{Input:} Observed history $\mathcal{X}_p$,
    future temporal features $\mathbf{y}^{c}$,
    prediction length $\tau$
    \STATE \textbf{Trained components:}
    Decoder $\mathcal{G}_{\phi}$,
    context encoder $\mathcal{H}_{\omega}$,
    vector field $\mathbf{v}_{\theta}$,
    prior parameters
    $\mathcal{P}=\{(\pi_k,\mu_k,B_k,\sigma_k)\}_{k=1}^{K}$
    \STATE \textbf{Parameters:} Number of forecast samples $N$,
    downsampling ratio $\rho$, local rank $r$,
    perturbation scale $\gamma$,
    integration grid $0=t_0<\cdots<t_M=1$
    \STATE \textbf{Output:}
    Forecast samples
    $\{\hat{\mathcal{X}}_t^{(n)}\}_{n=1}^{N}$

    \STATE
    \STATE \textcolor{blue}{\textbf{// Encode forecasting conditions}}
    \STATE $L\gets\lfloor\tau/\rho\rfloor$
    \STATE $\mathbf{c}\gets
    (\mathcal{H}_{\omega}(\mathcal{X}_p),\mathbf{y}^{c})$

    \FOR{$n=1$ to $N$}
        \STATE
        \STATE \textcolor{blue}{\textbf{// Sample the source latent sequence}}
        \STATE Independently sample
        $s_{\ell}^{0}\sim\mathrm{Cat}(\pi_1,\ldots,\pi_K)$
        and $\varepsilon_{\ell}\sim\mathcal{N}(0,I_r)$
        for $\ell=1,\ldots,L$
        \STATE $\mathbf{x}_{s,\ell}^{0}\gets
        \mathrm{Norm}\!\left(
        \mu_{s_{\ell}^{0}}
        +\gamma B_{s_{\ell}^{0}}
        (\sigma_{s_{\ell}^{0}}\odot\varepsilon_{\ell})
        \right)$
        \STATE $\hat{\mathbf{x}}_s^{t_0}\gets
        \operatorname{StackRows}
        (\mathbf{x}_{s,1}^{0},\ldots,\mathbf{x}_{s,L}^{0})$

        \STATE
        \STATE \textcolor{blue}{\textbf{// Joint latent transport}}
        \FOR{$j=0$ to $M-1$}
            \STATE $\Delta t_j\gets t_{j+1}-t_j$
            \STATE $\hat{\mathbf{x}}_s^{t_{j+1}}\gets
            \hat{\mathbf{x}}_s^{t_j}
            +\Delta t_j\,
            \mathbf{v}_{\theta}
            (\hat{\mathbf{x}}_s^{t_j},t_j\mid\mathbf{c})$
        \ENDFOR

        \STATE
        \STATE \textcolor{blue}{\textbf{// Decode the forecast}}
        \STATE $\hat{\mathcal{X}}_t^{(n)}\gets
        \mathcal{G}_{\phi}(\hat{\mathbf{x}}_s^{t_M})$
    \ENDFOR

    \RETURN $\{\hat{\mathcal{X}}_t^{(n)}\}_{n=1}^{N}$
\end{algorithmic}
\end{algorithm}

\noindent\textbf{Visualization of Forecasting Results.}
Figure~\ref{fig:all_visualization} presents forecasting examples
from five datasets with varying dimensionalities.
In these visualizations, ProtoFlow generally follows
the main temporal variations in both relatively smooth
and more fluctuating series, illustrating its ability
to accommodate different temporal patterns.
\begin{figure*}[h!]
\begin{center}
   \includegraphics[width=1.0\textwidth]{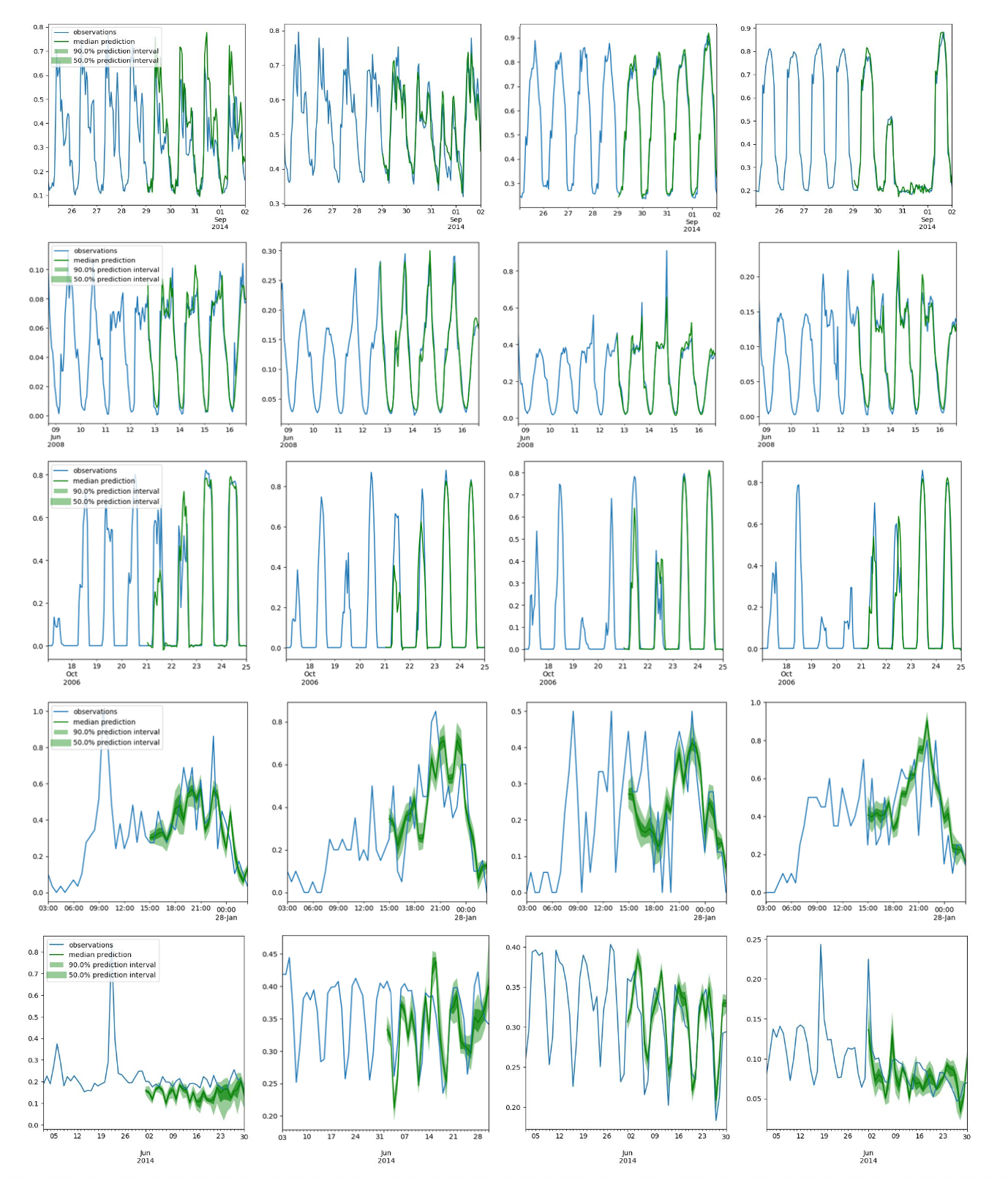}
  \vspace{-0.1in}
  \caption{ProtoFlow forecasts with a prediction horizon of 96. Rows show Electricity (370 dimensions), Traffic (963 dimensions), Solar (137 dimensions), Taxi (1214 dimensions), and Wikipedia (2000 dimensions)
from top to bottom, with four examples per dataset.}
  \label{fig:all_visualization}
\end{center}
\end{figure*}

\section{Limitations}
ProtoFlow constructs its prototype prior from global
training statistics and keeps it fixed during forecasting.
Although the learned transport is conditioned on
historical observations, the source prior does not
adapt its probability allocation to individual
forecasting contexts. A promising direction is to
develop context-dependent prototype weighting,
allowing initialization to reflect the current
temporal regime while retaining the efficiency
of few-step generation. We leave this direction for our future work.  
\end{document}